\documentclass{article}

\PassOptionsToPackage{square,numbers,sort&compress}{natbib}
\usepackage[utf8]{inputenc} % allow utf-8 input
\usepackage[T1]{fontenc}    % use 8-bit T1 fonts
\usepackage{booktabs}       % professional-quality tables
\usepackage{amsfonts}       % blackboard math symbols
\usepackage{nicefrac}       % compact symbols for 1/2, etc.
\usepackage{microtype}      % microtypography
\usepackage{xcolor}         % colors
\usepackage{xspace}
\usepackage{graphicx}
\usepackage{wrapfig}
\usepackage{hyperref} 

\usepackage{amsmath, amsthm, amssymb}
\usepackage{algorithm,algorithmicx}
\usepackage{algpseudocode}
\usepackage{multirow} 
\usepackage{subcaption} 
\usepackage{caption}
\usepackage{float}  
\usepackage{algorithm}
\usepackage{algpseudocode}

\usepackage{enumitem}

\theoremstyle{plain}
\newtheorem{theorem}{Theorem}%[section]
\newtheorem{proposition}{Proposition}%[section]
\newtheorem{corollary}{Corollary}%[section]
\newtheorem{innercustomthm}{Theorem}
\newenvironment{customthm}[1]
  {\renewcommand\theinnercustomthm{#1}\innercustomthm}
  {\endinnercustomthm}

\newtheorem{innercustomcor}{Corollary}
\newenvironment{customcor}[1]
  {\renewcommand\theinnercustomcor{#1}\innercustomcor}
  {\endinnercustomcor}

\newtheorem{innercustomprop}{Proposition}
\newenvironment{customprop}[1]
  {\renewcommand\theinnercustomprop{#1}\innercustomprop}
  {\endinnercustomprop}

\newtheorem*{theorem*}{Theorem}
\newtheorem*{proposition*}{Proposition}
\newtheorem*{lemma*}{Lemma}
\newtheorem*{property*}{Property}
\newtheorem*{definition*}{Definition}
\newtheorem*{corollary*}{Corollary}

\theoremstyle{definition}

\newcommand{\vparagraph}[1]{
  {\noindent \textbf{#1}}
}

\usepackage[preprint]{lib/neurips_2026}

\title{Interference-Aware Multi-Task Unlearning}

\author{Ying-Hua Huang \\
National Taiwan University \\
\texttt{yhhuang@arbor.ee.ntu.edu.tw} \\
\And Rui Fang \\
National Taiwan University \\
\texttt{rfang@arbor.ee.ntu.edu.tw} \\
\And
Hsi-Wen Chen \\
National Taiwan University \\
\texttt{hwchen@arbor.ee.ntu.edu.tw} \\
\And
Ming-Syan Chen \\
National Taiwan University \\
\texttt{mschen@ntu.edu.tw} \\
}

\begin{document}

\maketitle
\begin{abstract}
Machine unlearning aims to remove the contribution of designated training data from a trained model while preserving performance on the remaining data. Existing work mainly focuses on single-task settings, whereas modern models often operate in multi-task setups with shared backbones, where removing supervision for one task or instance can unintentionally affect others. We introduce \textbf{multi-task unlearning} with two settings: \textit{full-task unlearning}, which removes a target instance from all tasks, and \textit{partial-task unlearning}, which removes supervision only from selected tasks. We show that shared parameters couple the forget and retain sets, causing \textit{task-level interference} on non-target tasks and \textit{instance-level interference} on other instances. To address this issue, we propose an interference-aware framework that combines task-aware gradient projection, which constrains updates within task-specific subspaces, with instance-level gradient orthogonalization, which reduces conflicts between forget and retain signals. Experiments on two multi-task computer vision benchmarks across five tasks show that our method achieves effective unlearning while maintaining strong generalization, reducing UIS compared with the strongest baseline by $30.3\%$ in full-task unlearning and $52.9\%$ in partial-task unlearning. 
% The source code is available at \url{https://anonymous.4open.science/r/MultiTaskUnlearning-4D44}.
\end{abstract}
\section{Introduction}

Machine unlearning~\cite{cao2015towards} has become increasingly important as modern machine learning systems are required to remove sensitive or outdated information from trained models. This need arises from privacy regulations such as the General Data Protection Regulation (GDPR)~\cite{gdpr2018general}, as well as broader concerns in security~\cite{huang2025survey}, fairness~\cite{zhang2024forgotten}, and robustness~\cite{qian2023towards}. Beyond regulatory requirements, machine unlearning also supports practical applications such as debiasing~\cite{chen2023fast}, debugging~\cite{surve2025explaining}, and auditing~\cite{wang2025tape}. Machine unlearning partitions the training data into two subsets: the \emph{forget set}, whose influence should be removed, and the \emph{retain set}, whose performance should be preserved. The goal is to eliminate the influence of the forget set while maintaining performance on the retain set.

Ideally, unlearning should match retraining on the retain set~\cite{bourtoule2021machine}. However, retraining from scratch is computationally expensive for large models, motivating efficient unlearning methods~\cite{graves2021amnesiac,guo2019certified,golatkar2020eternal,ding2024unified,cha2024towards}. Existing methods primarily focus on the \emph{single-task} setting, whereas modern models are typically built on pretrained backbones and adapted to multiple tasks through shared representations~\cite{mahabadi2021parameter,liu2022polyhistor} or parameter-efficient adapters~\cite{kamalesh2024unolora,xin2024vmt,agiza2024mtlora}. In such multi-task settings, removing supervision for one task may unintentionally affect others, introducing challenges absent from single-task unlearning.

Therefore, we propose the \textbf{multi-task unlearning} problem, where a single input instance may be associated with multiple tasks. As shown in Fig.~\ref{fig:overview}, we consider two complementary setups: \textbf{full-task unlearning}, which removes a target instance from all tasks, and \textbf{partial-task unlearning}, which removes supervision for a target instance only from selected tasks. For example, an image may be removed from person identification~\cite{choi2023towards} due to privacy requirements while retained for action recognition~\cite{kong2022human}. Similarly, a user's interaction may be removed from personalized recommendation~\cite{li2025survey} while retained for fraud detection~\cite{du2019lifelong}.

However, our preliminary experiments show that directly applying single-task unlearning methods to multi-task models leads to substantial performance degradation, with up to a 25\% drop on the retain set. We attribute this degradation to interactions between the forget set and retained data through shared parameters. These interactions induce two types of interference: \textbf{task-level interference}, where unlearning affects tasks outside the target set, and \textbf{instance-level interference}, where unlearning a target instance degrades performance on other instances.

Based on these observations, we propose a multi-task unlearning framework that mitigates interference across tasks while preserving performance on retained data.\footnote{An alternative is to use task-specific adapters and remove the corresponding module for unlearning. However, this does not guarantee complete forgetting, as task information may remain in the shared backbone. Moreover, maintaining separate adapters becomes costly as the number of tasks grows.} Our framework consists of two key components, as illustrated in the right panel of Fig.~\ref{fig:overview}. First, \textbf{Task-Aware Gradient Projection} constrains parameter updates to task-specific subspaces, reducing unintended interference in the shared representation. Second, \textbf{Instance-Level Gradient Orthogonalization} removes conflicting components between forget and retain gradients, preventing degradation on retained instances. Together, these components mitigate both task- and instance-level interference through subspace-constrained and conflict-aware updates.

\textbf{Our contributions are summarized as follows.}
First, we introduce the multi-task unlearning problem with two settings, full-task and partial-task unlearning, enabling fine-grained control over data removal for privacy and utility. Second, we identify task-level and instance-level interference as two sources of degradation, and propose an interference-aware framework that combines Task-Aware Gradient Projection and Instance-Level Gradient Orthogonalization to mitigate them. Finally, across two benchmarks and five tasks, our method outperforms six baselines, reducing UIS by $30.3\%$ in FU and $52.9\%$ in PU while preserving retained performance.
\begin{figure}[t]
    \centering
    \includegraphics[width=0.9\columnwidth]{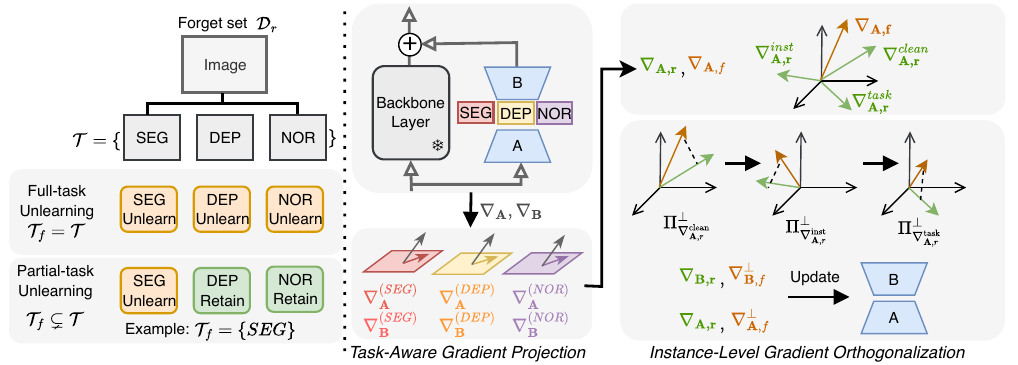}
    \caption{Overview of multi-task unlearning.}
    \label{fig:overview}
\end{figure}

\section{Problem Formulation}
\label{sec:problem_formulation}

In this paper, we study multi-task unlearning on a shared backbone, where each instance contains supervision for multiple tasks. Appendix~\ref{apx:notation} summarizes the key notations.

\paragraph{Multi-task Learning.}
Let $\mathcal{X}=\{\mathbf{x}_i\}^{N}_{i=1}$ denote the set of input instances such that $\mathbf{x}_i \in \mathbb{R}^{d}$, and let $\mathcal{T}=\{1,2,\dots,K\}$ denote the task set. For each instance $\mathbf{x}_i\in\mathcal{X}$ and task $t\in\mathcal{T}$, let $y_i^{(t)}$ denote the supervision signal for task $t$ on $\mathbf{x}_i$. The multi-task dataset is
$\mathcal{D} = \{(\mathbf{x}_i, t, y_i^{(t)}) \mid \mathbf{x}_i\in\mathcal{X},\ t\in\mathcal{T}\}$.
Accordingly, the loss for instance $\mathbf{x}_i$ on task $t$ is
\begingroup\small\begin{equation}
\ell_{i,t}(\theta)
:=
\ell_t\bigl(f_t(\mathbf{x}_i;\theta), y_i^{(t)}\bigr),
\end{equation}\endgroup
where $\ell_t$ is the task-specific loss, $f_t(\cdot;\theta)$ is the predictor for task $t$, and $\theta$ denotes the shared model parameters. Then, multi-task learning aims to optimize
\begingroup\small\begin{equation}
\theta^\star
=
\arg\min_{\theta}
\sum_{\mathbf{x}_i\in\mathcal{X}}\sum_{t\in\mathcal{T}}
\lambda_t\,\ell_{i,t}(\theta),
\end{equation}\endgroup
where $\lambda_t$ is the weight of task $t$. Because all tasks are learned through shared parameters $\theta^\star$, supervision from one task-instance pair can affect representations used by other tasks, making both learning and subsequent unlearning more challenging.

\paragraph{Multi-task Unlearning.}
Let $\mathcal{X}_f \subseteq \mathcal{X}$ denote the set of instances to be forgotten and $\mathcal{T}_f \subseteq \mathcal{T}$ denote the set of tasks whose supervision should be removed. The retained instance set and retained task set are defined as $\mathcal{X}_r = \mathcal{X} \setminus \mathcal{X}_f$ and $\mathcal{T}_r = \mathcal{T} \setminus \mathcal{T}_f$, respectively.

These two axes induce the following four partitions of the dataset:
\begingroup\small\begin{equation}
\begin{aligned}
\mathcal{D}_f
&= \{(\mathbf{x}_i, t, y_i^{(t)}) \mid \mathbf{x}_i\in\mathcal{X}_f,\ t\in\mathcal{T}_f\}, \quad
\mathcal{D}_r^{\mathrm{task}}
= \{(\mathbf{x}_i, t, y_i^{(t)}) \mid \mathbf{x}_i\in\mathcal{X}_f,\ t\in\mathcal{T}_r\}, \\
\mathcal{D}_r^{\mathrm{inst}}
&= \{(\mathbf{x}_i, t, y_i^{(t)}) \mid \mathbf{x}_i\in\mathcal{X}_r,\ t\in\mathcal{T}_f\}, \quad
\mathcal{D}_r^{\mathrm{clean}}
= \{(\mathbf{x}_i, t, y_i^{(t)}) \mid \mathbf{x}_i\in\mathcal{X}_r,\ t\in\mathcal{T}_r\}.
\end{aligned}
\end{equation}\endgroup

Here, $\mathcal{D}_f$ is the forget set, containing supervision on forgotten instances for forgotten tasks. The remaining three subsets are retained: $\mathcal{D}_r^{\mathrm{task}}$ keeps supervision on forgotten instances for retained tasks, $\mathcal{D}_r^{\mathrm{inst}}$ keeps supervision on retained instances for forgotten tasks, and $\mathcal{D}_r^{\mathrm{clean}}$ keeps supervision on retained instances for retained tasks. Accordingly, the retain set becomes
\begingroup\small\begin{equation}
\mathcal{D}_r
=
\mathcal{D}\setminus\mathcal{D}_f
=
\mathcal{D}_r^{\mathrm{task}}
\cup
\mathcal{D}_r^{\mathrm{inst}}
\cup
\mathcal{D}_r^{\mathrm{clean}}.
\end{equation}\endgroup
The objective is to remove the influence of $\mathcal{D}_f$ from $\theta^\star$ while preserving performance on $\mathcal{D}_r$.

Based on this formulation, we consider two practical scenarios of multi-task unlearning. When $\mathcal{T}_f=\mathcal{T}$, all supervision associated with the target instances is removed across all tasks, which we call \emph{full-task unlearning}. When $\mathcal{T}_f\subsetneq\mathcal{T}$, only supervision for a subset of tasks is removed, which we call \emph{partial-task unlearning}. This formulation provides fine-grained control, allowing an instance to be forgotten for selected tasks while retained for others.\footnote{While this paper focuses on removing supervision for selected instances, another possible setting removes supervision for one task across all instances, i.e., $\mathcal{X}_f=\mathcal{X}$, causing the model to forget that task capability entirely.}
\section{Theoretical Motivation: Interference in Multi-Task Unlearning}
\label{sec:theory}

Unlearning in the multi-task setting is particularly challenging because removing supervision from $\mathcal{D}_f$ may unintentionally affect retained supervision through shared model parameters. Specifically, \emph{task-level interference} arises when unlearning degrades performance on $\mathcal{D}_r^{\mathrm{task}}$, where forgotten instances should still be retained for non-target tasks. In contrast, \emph{instance-level interference} arises when unlearning degrades performance on $\mathcal{D}_r^{\mathrm{inst}}$, where other retained instances should still be preserved for the target tasks.

To formally characterize the interference induced by unlearning, we define the empirical losses on the retain and forget sets as
\begingroup\small\begin{equation}
L_r(\theta)
=
\frac{1}{|\mathcal{D}_r|}
\sum_{(\mathbf{x}_i,t,y_i^{(t)})\in\mathcal{D}_r}
\ell_{i,t}(\theta),
\qquad
L_f(\theta)
=
\frac{1}{|\mathcal{D}_f|}
\sum_{(\mathbf{x}_i,t,y_i^{(t)})\in\mathcal{D}_f}
\ell_{i,t}(\theta).
\end{equation}\endgroup
We then analyze the retrained model $\theta_r$, which represents the ideal solution obtained by unlearning $\mathcal{D}_f$ and retraining on $\mathcal{D}_r$.

\begin{theorem}
\label{thm:interference}
Assume that $L_r$ and $L_f$ are twice differentiable, $\mathbf{H}_r := \nabla^2 L_r(\theta_r)$ is invertible, and $|\mathcal{D}_r| \gg |\mathcal{D}_f|$. Let $\rho := |\mathcal{D}_f|/|\mathcal{D}_r|$, and suppose that $\theta^\star$ locally minimizes $L_r(\theta)+\rho L_f(\theta)$. Then, removing $\mathcal{D}_f$ induces the following loss change for any retained task-instance pair $(\mathbf{x}_i,t,y_i^{(t)})\in\mathcal{D}_r$:
\begingroup\small
\begin{equation*}
\ell_{i,t}(\theta_r)-\ell_{i,t}(\theta^\star)
=
\underbrace{
\rho\,
\nabla \ell_{i,t}(\theta_r)^\top
\mathbf{H}_r^{-1}
\nabla L_f(\theta_r)
}_{\text{first-order}}
+
\underbrace{
O(\rho^2)
}_{\text{higher-order}}.
\end{equation*}
\endgroup
\end{theorem}

Theorem~\ref{thm:interference} shows that joint interference in multi-task unlearning arises from Hessian-preconditioned gradient coupling between the forget set and retained data. Removing $\mathcal{D}_f$ induces a parameter shift governed by $\mathbf{H}_r^{-1}$, which reflects the local curvature of the retain loss. Since $|\mathcal{D}_r| \gg |\mathcal{D}_f|$ implies $\rho \ll 1$, the higher-order term $O(\rho^2)$ becomes negligible, and the first-order term captures the dominant source of interference.

These observations immediately yield the following corollary.

\begin{corollary}
\label{cor:interference}
Task-level and instance-level interference correspond to aggregations over $\mathcal{D}_r^{\mathrm{task}}$ and $\mathcal{D}_r^{\mathrm{inst}}$, respectively, and are both governed by the same first-order term in Theorem~\ref{thm:interference}.
\end{corollary}

Corollary~\ref{cor:interference} shows that the same first-order term governs both forms of interference along two axes. Aggregation over $\mathcal{D}_r^{\mathrm{task}}$ captures how unlearning affects forgotten instances on retained tasks, corresponding to task-level interference. Aggregation over $\mathcal{D}_r^{\mathrm{inst}}$ captures how unlearning affects retained instances on forgotten tasks, corresponding to instance-level interference.

Moreover, since the retraining-consistent parameter shift is shaped by the retain-set curvature through $\mathbf{H}_r^{-1}$, directly updating the model based solely on the unlearning gradient can be suboptimal. The following proposition formalizes this observation.

\begin{proposition}
\label{prop:suboptimal}
Under the local quadratic approximation
$L_r(\theta_r+\delta)\approx L_r(\theta_r)+\frac{1}{2}\delta^\top \mathbf{H}_r\delta$, assume that $\mathbf{H}_r$ is positive definite. Consider updates $\delta$ that achieve a fixed first-order unlearning effect, $\nabla L_f(\theta_r)^\top \delta = \gamma$, for some desired level $\gamma>0$. Then, the unique retain-loss-minimizing update is
\begingroup\small\begin{equation*}
\delta^\star
=
\frac{\gamma}
{\nabla L_f(\theta_r)^\top \mathbf{H}_r^{-1}\nabla L_f(\theta_r)}
\mathbf{H}_r^{-1}\nabla L_f(\theta_r).
\end{equation*}\endgroup
Hence, directly following the unlearning gradient, i.e., $\delta=\alpha \nabla L_f(\theta_r)$, is generally suboptimal unless $\nabla L_f(\theta_r)$ is an eigenvector of $\mathbf{H}_r$.
\end{proposition}

Here, the constraint $\nabla L_f(\theta_r)^\top \delta = \gamma$ enforces the same first-order unlearning effect across candidate updates, where $\gamma>0$ denotes the desired increase in forget loss. This is consistent with the unlearning direction, since $\delta$ increases the forget loss to first order. The eigenvector case is rare, as it requires the unlearning gradient to align with an eigenvector of $\mathbf{H}_r$. In general, $\nabla L_f(\theta_r)$ is determined by the forget set, while $\mathbf{H}_r$ reflects the retain-set curvature, making such alignment unlikely.
\section{Method}
We propose an interference-aware framework for multi-task unlearning that combines \textbf{Task-Aware Gradient Projection} with \textbf{Instance-Level Gradient Orthogonalization}. The key idea is to mitigate interference at two levels. First, we constrain updates to task-specific subspaces via gradient projection, reducing cross-task interference in the shared parameter space. Second, we remove conflicting components between forget and retain gradients through orthogonalization, preventing unlearning from degrading retained knowledge. These components are unified under a joint optimization objective that balances forgetting and retention while mitigating both task- and instance-level interference. Detailed pseudocode is provided in Appendix~\ref{apx:code}.

\subsection{Task-Aware Gradient Projection}

The goal of machine unlearning is to make the model behave as if it were trained only on the retain set $\mathcal{D}_r$, without exposure to the forget set $\mathcal{D}_f$. However, retraining from scratch for each unlearning request is computationally prohibitive~\cite{chowdhury2025towards,sekhari2021remember,liu2024machine}. Thus, existing approaches start from a pretrained model $\theta^\star$ and apply efficient updates to remove information associated with $\mathcal{D}_f$, while preserving performance on $\mathcal{D}_r$ and maintaining generalization~\cite{kurmanji2023towards,graves2021amnesiac,guo2019certified,golatkar2020eternal}.

To facilitate analysis, we consider a single layer with weight matrix $\mathbf{W}^\star \in \mathbb{R}^{d \times k}$ from $\theta^\star$. A natural approach is to update $\mathbf{W}^\star$ via gradient-based optimization, incorporating a reverse gradient direction induced by the forget set $\mathcal{D}_f$ to facilitate unlearning and a forward gradient direction sampled from the retain set $\mathcal{D}_r$ for calibration. However, such strategies become costly in multi-task settings due to task diversity and repeated calibration overhead~\cite{bourtoule2021machine,yu2020gradient}. Moreover, directly updating the full parameter matrix remains computationally expensive and impractical at scale~\cite{zhao2024continual,shamsian2025go,poppi2024unlearning}.

To address this issue, we reformulate machine unlearning as a model editing problem by learning a plug-in parameter that modifies model behavior. Instead of iteratively updating the full model parameters, we adopt a parameter-efficient formulation and represent the unlearning process as a low-rank update:
\begingroup\small\begin{equation}
\widetilde{\mathbf{W}} = \mathbf{W}^\star + \mathbf{B}\mathbf{A}^\top,
\label{eq:edit}
\end{equation}\endgroup
where $\mathbf{A} \in \mathbb{R}^{k \times r}$ and $\mathbf{B} \in \mathbb{R}^{d \times r}$ are learnable factors, while $\mathbf{W}^\star$ remains fixed. One can update $\mathbf{A}$ and $\mathbf{B}$ using gradients induced by both forget and retain supervision signals. However, this is suboptimal in multi-task settings, as directions associated with $\mathcal{D}_f$ are often entangled with directions useful for $\mathcal{D}_r$. This issue is particularly pronounced in partial unlearning, where supervision for selected tasks is removed while supervision for retained tasks on the same instances must be preserved, leading to \emph{task-level interference}.

To enable selective unlearning, we introduce task-aware gradient projection, which restricts updates for task $t$ to a task-specific subspace $\mathbf{P}_t$:
\begingroup\small\begin{equation}
\nabla_{\mathbf{A}}^{(t)} := \nabla_{\mathbf{A}}\mathbf{P}_t,
\qquad
\nabla_{\mathbf{B}}^{(t)} := \nabla_{\mathbf{B}}\mathbf{P}_t.
\label{eq:projected_gradients}
\end{equation}\endgroup
This projection confines updates to task-specific subspaces, reducing cross-task interference in the shared low-rank adaptation space~\cite{yang2026disentangling,yu2020gradient,liu2021conflict,saha2021gradient}.

Since task-specific updates lie in subspaces of the shared low-rank space, each task occupies only part of the $r$-dimensional adaptation space. Accordingly, for each task $t$, we define an orthonormal task-specific basis
\begingroup\small\begin{equation}
\mathbf{U}_t \in \mathbb{R}^{r \times s},
\qquad
\mathbf{U}_t^\top \mathbf{U}_t = \mathbf{I}_s,
\qquad
\mathbf{P}_t := \mathbf{U}_t \mathbf{U}_t^\top \in \mathbb{R}^{r \times r},
\end{equation}\endgroup
where $r$ is the shared low-rank dimension and $s \le r$ denotes the dimension of the task-specific subspace. Here, $\mathbf{P}_t$ is the orthogonal projector induced by $\mathbf{U}_t$.

During training, we further regularize the task-specific subspaces to be mutually orthogonal. Specifically, for any $t \neq t'$, we minimize the alignment between task-specific subspaces via $\|\mathbf{U}_t^\top \mathbf{U}_{t'}\|_F^2$, which promotes separation between task-specific directions~\cite{wang2023orthogonal,yifei2025dislora}. Consequently, each task updates the shared low-rank factors $\mathbf{A}$ and $\mathbf{B}$ through its corresponding projector $\mathbf{P}_t$. This confines updates to weakly aligned task-specific subspaces, reduces alignment between projected task gradients, and helps mitigate cross-task interference. We formalize this effect in Theorem~\ref{thm:task_projection} in Appendix~\ref{proof:task_projection}.
\subsection{Instance-Level Gradient Orthogonalization}

While task-aware gradient projection restricts updates to task-specific subspaces, it does not eliminate conflicts between forgetting and retention. The forget gradient may still align with directions important for retained data, thereby degrading retain performance and overall generalization~\cite{kurmanji2023towards,graves2021amnesiac,golatkar2020eternal,guo2019certified}.

Let $\mathbf{Z}\in\{\mathbf{A},\mathbf{B}\}$ denote either of the trainable matrices. To prevent such interference when updating $\mathbf{Z}$, we remove the component of the forget gradient $\nabla_{\mathbf{Z},f}$ that aligns with the retain gradient $\nabla_{\mathbf{Z},r}$ via orthogonal projection:
\begingroup\small\begin{equation}
\label{eq:orthogonalization}
\Pi_{\nabla_{\mathbf{Z},r}}^{\perp}(\nabla_{\mathbf{Z},f})
=
\nabla_{\mathbf{Z},f}
-
\frac{\langle \nabla_{\mathbf{Z},f}, \nabla_{\mathbf{Z},r}\rangle_F}
{\|\nabla_{\mathbf{Z},r}\|_F^2+\varepsilon}\,\nabla_{\mathbf{Z},r},
\end{equation}\endgroup
where $\varepsilon>0$ ensures numerical stability. All gradients are already projected onto task-specific subspaces as defined in Eq.~\eqref{eq:projected_gradients}. This operation removes the retain-aligned component of the forget gradient up to the stabilization term $\varepsilon$, making the forget update less disruptive to retained knowledge. We analyze this effect in Theorem~\ref{thm:instance_orthogonalization} of Appendix~\ref{proof:instance_orthogonalization}.

However, multi-task unlearning is more challenging because retained supervision comes from multiple sources, including $\mathcal{D}_r^{\mathrm{clean}}$, $\mathcal{D}_r^{\mathrm{inst}}$, and $\mathcal{D}_r^{\mathrm{task}}$. These subsets impose heterogeneous retention constraints that must be preserved simultaneously. Thus, uniformly sampling from $\mathcal{D}_r$ may fail to capture these constraints effectively~\cite{hoang2024learn,shamsian2025go,lin2024gdr}.

At each iteration, we decompose $\nabla_{\mathbf{Z},r}$ into three components, namely $\nabla_{\mathbf{Z},r}^{\mathrm{clean}}$, $\nabla_{\mathbf{Z},r}^{\mathrm{inst}}$, and $\nabla_{\mathbf{Z},r}^{\mathrm{task}}$, corresponding to the three retained subsets. We then adopt a sequential orthogonalization scheme that applies the projection operators successively:
\begingroup\small\begin{equation}
\label{eq:retain}
\nabla_{\mathbf{Z},f}^{\perp}
:=
\Pi_{\nabla_{\mathbf{Z},r}^{\mathrm{task}}}^{\perp}
\Bigl(
\Pi_{\nabla_{\mathbf{Z},r}^{\mathrm{inst}}}^{\perp}
\bigl(
\Pi_{\nabla_{\mathbf{Z},r}^{\mathrm{clean}}}^{\perp}(\nabla_{\mathbf{Z},f})
\bigr)
\Bigr).
\end{equation}\endgroup

We process retain gradients in the order of clean, instance-level, and task-level signals as a practical design. Clean retain signals provide stable global guidance, instance-level signals capture retention constraints on other instances, and task-level signals preserve supervision on non-target tasks for forgotten instances. This sequential scheme progressively removes components of the forget gradient that conflict with different retained subsets, mitigating both instance- and task-level interference.
\subsection{Overall Optimization}

The sequential orthogonalization step in Eq.~\eqref{eq:retain} produces retain-aware forget gradients $\nabla_{\mathbf{A},f}^{\perp}$ and $\nabla_{\mathbf{B},f}^{\perp}$ that avoid direct conflict with retained supervision. To further stabilize training, we optionally incorporate preservation terms derived from retain gradients~\cite{cheng2026machine,hoang2024learn,kurmanji2023towards}.

We update the low-rank factors by combining a retain-preserving descent direction with a forget-promoting ascent direction:
\begingroup\small\begin{equation}
\mathbf{A} \leftarrow \mathbf{A} - \eta_1 \nabla_{\mathbf{A},r} + \eta_2 \nabla_{\mathbf{A},f}^{\perp}, 
\qquad
\mathbf{B} \leftarrow \mathbf{B} - \eta_1 \nabla_{\mathbf{B},r} + \eta_2 \nabla_{\mathbf{B},f}^{\perp},
\label{eq:parameter_update}
\end{equation}\endgroup
where $\eta_1,\eta_2>0$ control the strengths of retention and forgetting, respectively. The first term decreases the retain loss, while the second increases the forget loss after removing components that conflict with retained gradients. When $\eta_1=0$, the update reduces to pure unlearning using only the projected forget gradients. Here, $\nabla_{\mathbf{A},r}$ and $\nabla_{\mathbf{B},r}$ aggregate retain gradients from $\mathcal{D}_r^{\mathrm{clean}}$, $\mathcal{D}_r^{\mathrm{inst}}$, and $\mathcal{D}_r^{\mathrm{task}}$. The pretrained weight $\mathbf{W}^\star$ is kept fixed throughout training. The final unlearned model is obtained by merging the learned update in Eq.~\eqref{eq:edit}, enabling efficient unlearning while mitigating both task-level and instance-level interference.

\section{Experiments}
\label{sec:experiments}

We evaluate our method under both \emph{full-task} and \emph{partial-task} unlearning settings on two vision benchmarks spanning five tasks. Our experiments examine four key requirements for multi-task unlearning: (1) effective removal of information from the forget set, (2) preservation of retained capabilities on the retain set, (3) mitigation of unintended interference and preservation of generalization, and (4) reduction of residual membership signals associated with forgotten data for privacy protection.

\subsection{Setup}

\vparagraph{Datasets.}
To evaluate multi-task unlearning across varying supervision granularities, we consider two representative vision benchmarks that span image-level, instance-level, and pixel-level prediction. \textbf{NYUv2}~\cite{silberman2012indoor} serves as a dense multi-task benchmark with pixel-level supervision, where three tasks are defined on the same images: \textit{semantic segmentation (SEG)}, \textit{depth estimation (DEP)}, and \textit{surface-normal prediction (NOR)}. \textbf{PASCAL}~\cite{Everingham10} captures heterogeneous supervision, encompassing image-level \textit{image classification (CLS)} and instance-level \textit{object detection (OD)}.

\vparagraph{Baselines.}
As reference models, we include \textbf{Original}, trained on the full dataset, and \textbf{Retrain}, trained from scratch on the retain set, which provide lower- and upper-bound references. For first-order gradient-based methods, we consider \textbf{NegGrad+}~\cite{kurmanji2023towards,graves2021amnesiac}, which performs gradient ascent on forget data with retain-aware optimization. For second-order approaches, we include \textbf{Fisher}~\cite{golatkar2020eternal} and \textbf{Influence}~\cite{guo2019certified}, which estimate parameter importance via Fisher information or influence functions. We further include \textbf{SSD}~\cite{foster2024fast}, which selectively dampens parameters associated with forget data, and \textbf{OrthoGrad}~\cite{shamsian2025go}, which orthogonalizes forget updates against retain gradients. Finally, we consider \textbf{SCRUB}~\cite{kurmanji2023towards}, which integrates forgetting and retention through a teacher-student framework.

\vparagraph{Task utility and privacy metrics.}
For task-specific utility, we report mean Intersection-over-Union (mIoU) for semantic segmentation~\cite{long2015fully}, Threshold Accuracy for depth estimation and Angular Accuracy for surface-normal prediction~\cite{silberman2012indoor}, and mean Average Precision (mAP) for both object detection and multi-label classification~\cite{lin2014microsoft,Everingham10}. Each metric is evaluated on three disjoint splits: the retain set (\textsc{Ret}), the forget set (\textsc{Unl}), and a held-out validation set (\textsc{Val}), corresponding to retention, forgetting, and generalization. The retain set measures preservation of useful supervision, the forget set quantifies the extent of forgetting, and the validation set evaluates generalization beyond the affected training instances. To assess potential privacy leakage, we also conduct Membership Inference Attack~\cite{shokri2017membership} (\textsc{MIA}) to quantify residual memorization after unlearning, i.e., whether a sample can still be inferred as part of the training set.

\vparagraph{Unlearning Impact Score (UIS).}
Evaluating multi-task unlearning is inherently multi-objective, requiring a balance among retention, forgetting, generalization, and memorization~\cite{tu2024towards,cheng2024remaining,ahmed2025towards,shamsian2025go}. Therefore, beyond raw metrics, we assess each method by its deviation from a desired reference behavior: the model should match the retrained model on supervision that should be forgotten, while preserving the original model behavior on supervision that should be retained. This reflects the goal of removing the influence of target supervision without degrading unrelated knowledge. For each task, we compute its relative deviation from the corresponding reference over the relevant evaluation aspects and aggregate the deviations into an overall discrepancy score. A smaller score indicates closer alignment with the desired behavior. In \textbf{full-task unlearning}, all tasks on forget instances are compared against the \textbf{retrained model}, while retain and validation behavior are compared against the \textbf{original model}. In \textbf{partial-task unlearning}, only the \textbf{selected forgotten task} is compared against the \textbf{retrained reference}, whereas the \textbf{remaining tasks} are compared against the \textbf{original model}.

\vparagraph{Implementation Details.}
We evaluate our framework with two representative backbones, ViT-L~\cite{dosovitskiy2021an} and Swin-L~\cite{liu2021swin} (presented in Appendix~\ref{apx:swin}), initialized from ImageNet-pretrained HuggingFace checkpoints~\cite{wolf2019huggingface,deng2009imagenet}. For NYUv2, we attach an ASPP module~\cite{chen2017rethinking} with task-specific heads for semantic segmentation, depth estimation, and surface-normal prediction. For PASCAL, we use a linear classification head and a DETR-style detection head with 100 object queries~\cite{carion2020end}. To model realistic multi-task sharing, we use a single shared LoRA~\cite{huang2023lorahub} across all tasks and follow the official training and validation splits~\cite{Everingham10,silberman2012indoor}. Following prior protocols~\cite{kurmanji2023towards,shamsian2025go}, we designate 10\% of the training instances as the forget set and use the rest as the retain set. Unlearning updates only LoRA parameters with AdamW~\cite{loshchilov2017decoupled} for up to 20 epochs with early stopping. All results are reported as the average of 10 runs, with standard deviations within 3\%.
\begin{table*}[t]
\centering
\small
\caption{
Quantitative results of multi-task unlearning on NYUv2 using ViT-L}
\label{tab:nyuv2_unlearning}

\resizebox{\textwidth}{!}{
\begin{tabular}{llccccccccccccc}
\toprule
\multirow{3}{*}{Setting} & \multirow{3}{*}{Method}
& \multicolumn{4}{c}{\textbf{SEG}}
& \multicolumn{4}{c}{\textbf{DEP}}
& \multicolumn{4}{c}{\textbf{NOR}}
& \multirow{3}{*}{\textbf{UIS}$\downarrow$} \\
\cmidrule(lr){3-6} \cmidrule(lr){7-10} \cmidrule(lr){11-14}
&
& \textsc{Ret} & \textsc{Unl} & \textsc{Val} & \textsc{MIA}
& \textsc{Ret} & \textsc{Unl} & \textsc{Val} & \textsc{MIA}
& \textsc{Ret} & \textsc{Unl} & \textsc{Val} & \textsc{MIA}
& \\
&
& (mIoU) & (mIoU) & (mIoU) & (AUC)
& ($\sigma_{1.25}$) & ($\sigma_{1.25}$) & ($\sigma_{1.25}$) & (AUC)
& (A$_{30}$) & (A$_{30}$) & (A$_{30}$) & (AUC)
& \\
\midrule

\multirow{2}{*}{\shortstack{Ref.}}
& Original
& 0.9349 & 0.9342 & 0.7555 & 0.9716
& 0.8573 & 0.8596 & 0.7130 & 0.6267
& 0.6713 & 0.6738 & 0.6045 & 0.6146
& -- \\
& Retrain
& 0.9400 & 0.7636 & 0.7515 & 0.5499
& 0.8687 & 0.6887 & 0.7027 & 0.4522
& 0.6698 & 0.5775 & 0.5958 & 0.4607
& -- \\
\midrule

\multirow{7}{*}{\shortstack{FU}}
& NegGrad+
& 0.8647 & 0.7611 & 0.7006 & 0.5421
& 0.8259 & 0.7480 & 0.6838 & 0.4851
& 0.5872 & 0.4729 & 0.4963 & 0.4413
& 30.5\% \\
& Fisher
& 0.5149 & 0.4883 & 0.4626 & 0.5539
& 0.6235 & 0.5601 & 0.5612 & 0.4529
& 0.4132 & 0.3921 & 0.3951 & 0.4936
& 99.6\% \\
& Influence
& 0.6736 & 0.6040 & 0.5773 & 0.5012
& 0.6293 & 0.5902 & 0.5529 & 0.4928
& 0.4802 & 0.4593 & 0.4492 & 0.4802
& 77.0\% \\
& SSD
& 0.7018 & 0.7113 & 0.6529 & 0.6938
& 0.4883 & 0.4517 & 0.4938 & 0.4120
& 0.4382 & 0.4158 & 0.3902 & 0.4937
& 97.5\% \\
& OrthoGrad
& 0.8767 & 0.7461 & 0.6973 & 0.5961
& 0.8123 & 0.6929 & 0.6728 & 0.4719
& 0.5828 & 0.4992 & 0.4983 & 0.4494
& 28.6\% \\
& SCRUB
& 0.8196 & 0.7003 & 0.6779 & 0.5170
& 0.8296 & 0.7241 & 0.6767 & 0.4997
& 0.5865 & 0.5403 & 0.4877 & 0.5251
& 37.2\% \\
& Ours
& 0.8728 & 0.7543 & 0.7028 & 0.5197
& 0.8556 & 0.7716 & 0.7063 & 0.4868
& 0.6217 & 0.5696 & 0.5444 & 0.4918
& \textbf{22.0\%} \\
\midrule

\multirow{7}{*}{\shortstack{PU\\(SEG)}}
& NegGrad+
& 0.8117 & 0.7065 & 0.6637 & 0.5617
& 0.8317 & 0.8256 & 0.6997 & 0.5631
& 0.5833 & 0.5761 & 0.5012 & 0.5890
& 34.3\% \\
& Fisher
& 0.5248 & 0.5107 & 0.4879 & 0.5702
& 0.7973 & 0.7901 & 0.6914 & 0.4517
& 0.5713 & 0.5624 & 0.4949 & 0.5516
& 74.0\% \\
& Influence
& 0.7459 & 0.7230 & 0.6128 & 0.6322
& 0.7393 & 0.7307 & 0.6193 & 0.5471
& 0.5629 & 0.5442 & 0.4856 & 0.5503
& 59.8\% \\
& SSD
& 0.7048 & 0.7030 & 0.6539 & 0.6928
& 0.4932 & 0.4501 & 0.5024 & 0.4221
& 0.4392 & 0.4137 & 0.4037 & 0.5193
& 115.4\% \\
& OrthoGrad
& 0.8202 & 0.7238 & 0.6683 & 0.5661
& 0.8229 & 0.8230 & 0.6892 & 0.5722
& 0.5937 & 0.5918 & 0.5027 & 0.5622
& 33.8\% \\
& SCRUB
& 0.8715 & 0.7510 & 0.6913 & 0.5502
& 0.8384 & 0.8217 & 0.6945 & 0.5427
& 0.6027 & 0.5716 & 0.5028 & 0.5692
& 29.7\% \\
& Ours
& 0.8602 & 0.7517 & 0.7073 & 0.5545
& 0.8499 & 0.8695 & 0.7085 & 0.6300
& 0.6316 & 0.6315 & 0.5562 & 0.5769
& \textbf{15.4\%} \\
\midrule

\multirow{7}{*}{\shortstack{PU\\(DEP)}}
& NegGrad+
& 0.9344 & 0.9270 & 0.7477 & 0.9524
& 0.8312 & 0.7568 & 0.6965 & 0.4684
& 0.6052 & 0.5971 & 0.5091 & 0.5649
& 22.5\% \\
& Fisher
& 0.9310 & 0.9257 & 0.7502 & 0.9623
& 0.2895 & 0.2301 & 0.3392 & 0.4603
& 0.6012 & 0.5832 & 0.5023 & 0.5732
& 79.1\% \\
& Influence
& 0.9213 & 0.9202 & 0.7412 & 0.9543
& 0.6103 & 0.5742 & 0.5385 & 0.4696
& 0.5847 & 0.5823 & 0.5138 & 0.5623
& 43.4\% \\
& SSD
& 0.7830 & 0.7827 & 0.7305 & 0.7739
& 0.5819 & 0.5632 & 0.5701 & 0.4239
& 0.4829 & 0.4605 & 0.4391 & 0.5281
& 77.9\% \\
& OrthoGrad
& 0.9292 & 0.9283 & 0.7402 & 0.9586
& 0.8099 & 0.6737 & 0.6729 & 0.4420
& 0.6193 & 0.5927 & 0.5113 & 0.5720
& 20.7\% \\
& SCRUB
& 0.9377 & 0.9317 & 0.7587 & 0.9572
& 0.7497 & 0.6729 & 0.6229 & 0.4882
& 0.6028 & 0.5928 & 0.5129 & 0.5733
& 27.3\% \\
& Ours
& 0.9323 & 0.9410 & 0.7541 & 0.9787
& 0.8337 & 0.7126 & 0.6929 & 0.4523
& 0.6321 & 0.6313 & 0.5556 & 0.5780
& \textbf{12.3\%} \\
\midrule

\multirow{7}{*}{\shortstack{PU\\(NOR)}}
& NegGrad+
& 0.9397 & 0.9265 & 0.7410 & 0.9577
& 0.8460 & 0.8427 & 0.6968 & 0.5819
& 0.5491 & 0.4477 & 0.4627 & 0.4773
& 27.9\% \\
& Fisher
& 0.9132 & 0.9121 & 0.7305 & 0.9591
& 0.8325 & 0.8312 & 0.6839 & 0.5628
& 0.2738 & 0.2448 & 0.2675 & 0.5176
& 71.3\% \\
& Influence
& 0.9120 & 0.9185 & 0.7201 & 0.9651
& 0.8332 & 0.8102 & 0.6770 & 0.5700
& 0.5624 & 0.4611 & 0.4623 & 0.4627
& 30.4\% \\
& SSD
& 0.7038 & 0.7142 & 0.6521 & 0.7018
& 0.4831 & 0.4576 & 0.4928 & 0.4292
& 0.4392 & 0.4130 & 0.3917 & 0.5079
& 116.7\% \\
& OrthoGrad
& 0.9384 & 0.9211 & 0.7492 & 0.9682
& 0.8572 & 0.8401 & 0.7025 & 0.5821
& 0.5793 & 0.4729 & 0.4623 & 0.4729
& 23.5\% \\
& SCRUB
& 0.9291 & 0.9288 & 0.7422 & 0.9605
& 0.8294 & 0.7650 & 0.6887 & 0.5490
& 0.5900 & 0.4955 & 0.5045 & 0.4739
& 26.2\% \\
& Ours
& 0.9305 & 0.9463 & 0.7531 & 0.9858
& 0.8537 & 0.8757 & 0.7054 & 0.6573
& 0.6192 & 0.5486 & 0.5425 & 0.4783
& \textbf{12.4\%} \\
\bottomrule
\end{tabular}
}
\end{table*}
\begin{table*}[t]
\centering
\small
\caption{Quantitative results of multi-task unlearning on Pascal using ViT-L}
\label{tab:pascal_unlearning}
\resizebox{0.74\textwidth}{!}{
\begin{tabular}{llccccccccc}
\toprule
\multirow{3}{*}{Setting} & \multirow{3}{*}{Method}
& \multicolumn{4}{c}{\textbf{CLS}}
& \multicolumn{4}{c}{\textbf{OD}}
& \multirow{3}{*}{\textbf{UIS}$\downarrow$} \\
\cmidrule(lr){3-6} \cmidrule(lr){7-10}
&
& \textsc{Ret} & \textsc{Unl} & \textsc{Val} & \textsc{MIA}
& \textsc{Ret} & \textsc{Unl} & \textsc{Val} & \textsc{MIA}
& \\
&
& (mAP) & (mAP) & (mAP) & (AUC)
& (mAP) & (mAP) & (mAP) & (AUC)
& \\
\midrule

\multirow{2}{*}{\shortstack{Ref.}}
& Original
& 1.0000 & 1.0000 & 0.9245 & 0.6645
& 0.5144 & 0.5352 & 0.4101 & 0.7033
& -- \\
& Retrain
& 1.0000 & 0.9130 & 0.9055 & 0.5373
& 0.5131 & 0.4270 & 0.4223 & 0.4953
& -- \\
\midrule

\multirow{7}{*}{\shortstack{FU}}
& NegGrad+
& 0.9705 & 0.9588 & 0.8653 & 0.5284
& 0.3659 & 0.3126 & 0.3080 & 0.5370
& 52.5\% \\
& Fisher
& 0.8922 & 0.8479 & 0.8372 & 0.5134
& 0.3649 & 0.3037 & 0.2854 & 0.5571
& 66.3\% \\
& Influence
& 0.9077 & 0.8551 & 0.8324 & 0.5079
& 0.3781 & 0.2945 & 0.3139 & 0.5439
& 61.0\% \\
& SSD
& 0.9900 & 0.9663 & 0.9116 & 0.6018
& 0.4739 & 0.4237 & 0.3382 & 0.6127
& 35.8\% \\
& OrthoGrad
& 0.9621 & 0.9297 & 0.8193 & 0.5892
& 0.4592 & 0.4182 & 0.3928 & 0.6482
& 37.6\% \\
& SCRUB
& 0.9753 & 0.9182 & 0.8230 & 0.5439
& 0.4183 & 0.3301 & 0.3092 & 0.5563
& 46.8\% \\
& Ours
& 0.9979 & 0.9485 & 0.8663 & 0.5297
& 0.4451 & 0.3731 & 0.3803 & 0.4951
& \textbf{22.9\%} \\
\midrule

\multirow{7}{*}{\shortstack{PU\\(CLS)}}
& NegGrad+
& 0.9155 & 0.8440 & 0.7639 & 0.5829
& 0.3967 & 0.4180 & 0.3655 & 0.6258
& 53.4\% \\
& Fisher
& 0.8631 & 0.8289 & 0.8057 & 0.5567
& 0.3702 & 0.4293 & 0.3283 & 0.6695
& 55.1\% \\
& Influence
& 0.8157 & 0.7660 & 0.7320 & 0.5202
& 0.4007 & 0.3131 & 0.2938 & 0.6033
& 81.5\% \\
& SSD
& 0.8671 & 0.8432 & 0.8425 & 0.4861
& 0.3776 & 0.3851 & 0.2883 & 0.6071
& 67.7\% \\
& OrthoGrad
& 0.9721 & 0.9403 & 0.8351 & 0.6184
& 0.4036 & 0.3889 & 0.3288 & 0.6751
& 50.7\% \\
& SCRUB
& 0.9829 & 0.9548 & 0.9157 & 0.6283
& 0.4201 & 0.3838 & 0.3634 & 0.6217
& 47.0\% \\
& Ours
& 0.9701 & 0.8303 & 0.8558 & 0.5403
& 0.4569 & 0.5413 & 0.3992 & 0.7105
& \textbf{17.0\%} \\
\midrule

\multirow{7}{*}{\shortstack{PU\\(OD)}}
& NegGrad+
& 1.0000 & 0.9906 & 0.8875 & 0.6457
& 0.3084 & 0.2524 & 0.2507 & 0.5367
& 68.8\% \\
& Fisher
& 0.9541 & 0.9784 & 0.8673 & 0.6032
& 0.2847 & 0.2201 & 0.2157 & 0.4683
& 84.8\% \\
& Influence
& 0.9998 & 0.9992 & 0.8996 & 0.6307
& 0.3277 & 0.2804 & 0.2776 & 0.5488
& 61.7\% \\
& SSD
& 0.9963 & 0.9991 & 0.8704 & 0.6143
& 0.3085 & 0.2474 & 0.2567 & 0.5161
& 69.6\% \\
& OrthoGrad
& 1.0000 & 0.9999 & 0.8965 & 0.6523
& 0.4019 & 0.3521 & 0.3037 & 0.5639
& 43.0\% \\
& SCRUB
& 0.9970 & 0.9863 & 0.8759 & 0.5874
& 0.3892 & 0.3492 & 0.3251 & 0.5127
& 41.2\% \\
& Ours
& 1.0000 & 1.0000 & 0.9017 & 0.6684
& 0.4632 & 0.3738 & 0.3675 & 0.4962
& \textbf{19.2\%} \\
\bottomrule
\end{tabular}
}
\end{table*}

\subsection{Full-task Unlearning.}
Tables~\ref{tab:nyuv2_unlearning} and~\ref{tab:pascal_unlearning} report full-task unlearning (FU), where all task supervision on forget instances is removed. The unlearned model should move forget instances toward \textit{Retrain}, while keeping retain and validation performance close to \textit{Original}. UIS measures deviation from this target behavior, with smaller values indicating a better forgetting-preservation trade-off. Overall, our method achieves the lowest UIS in FU on both benchmarks, reducing the score from $28.6\%$ to $22.0\%$ on NYUv2 and from $35.8\%$ to $22.9\%$ on Pascal compared with the strongest baseline. These results show that our method removes forget supervision (\textsc{Unl}), preserves retain performance (\textsc{Ret}), maintains generalization (\textsc{Val}), and reduces privacy leakage (\textsc{MIA}) by modeling task- and instance-level interference. The average gain is larger on NYUv2, where dense prediction tasks share spatial structure, while Pascal is harder because classification and detection use different supervision granularities. A similar trend is observed with Swin-L, where UIS is reduced from $31.2\%$ to $14.6\%$, as shown in Appendix~\ref{apx:swin}.

\subsection{Partial-task Unlearning}
Tables~\ref{tab:nyuv2_unlearning} and~\ref{tab:pascal_unlearning} also report partial-task unlearning (PU), where only selected task labels on forget instances are removed. In the tables, PU(SEG), PU(DEP), and PU(NOR) denote unlearning semantic segmentation, depth estimation, and surface normal estimation on NYUv2, while PU(CLS) and PU(OD) denote unlearning classification and object detection on Pascal. This setting requires finer control because the same instance may contain both forgotten and retained supervision. The model should move the forgotten task toward \textit{Retrain} while keeping the remaining tasks close to \textit{Original}. Overall, our method achieves the lowest UIS across PU settings, reducing UIS from $24.6\%$ to $13.4\%$ on NYUv2 and from $44.1\%$ to $18.1\%$ on Pascal, corresponding to $45.7\%$ and $59.0\%$ relative reductions. In contrast, SSD and Fisher reduce retained DEP and NOR performance under NYUv2 PU(SEG), while Influence, SSD, and Fisher hurt retained or validation OD performance under Pascal PU(CLS). These results show that existing methods often let forgetting updates spill over to non-target tasks, whereas our method better decouples task- and instance-level unlearning signals.

\subsection{Ablation Study}

\begin{table*}[t]
\centering
\small
\caption{
Ablation study for multi-task unlearning on NYUv2 using ViT-L.
}
\label{tab:nyuv2_ablation}

\resizebox{\textwidth}{!}{
\begin{tabular}{llccccccccccccc}
\toprule
\multirow{3}{*}{Setting} & \multirow{3}{*}{Method}
& \multicolumn{4}{c}{\textbf{SEG}}
& \multicolumn{4}{c}{\textbf{DEP}}
& \multicolumn{4}{c}{\textbf{NOR}}
& \multirow{3}{*}{\textbf{UIS}$\downarrow$} \\
\cmidrule(lr){3-6} \cmidrule(lr){7-10} \cmidrule(lr){11-14}
&
& \textsc{Ret} & \textsc{Unl} & \textsc{Val} & \textsc{MIA}
& \textsc{Ret} & \textsc{Unl} & \textsc{Val} & \textsc{MIA}
& \textsc{Ret} & \textsc{Unl} & \textsc{Val} & \textsc{MIA}
& \\
&
& (mIoU) & (mIoU) & (mIoU) & (AUC)
& ($\sigma_{1.25}$) & ($\sigma_{1.25}$) & ($\sigma_{1.25}$) & (AUC)
& (A$_{30}$) & (A$_{30}$) & (A$_{30}$) & (AUC)
& \\
\midrule

\multirow{2}{*}{\shortstack{Ref.}}
& Original
& 0.9349 & 0.9342 & 0.7555 & 0.9716
& 0.8573 & 0.8596 & 0.7130 & 0.6267
& 0.6713 & 0.6738 & 0.6045 & 0.6146
& -- \\
& Retrain
& 0.9400 & 0.7636 & 0.7515 & 0.5499
& 0.8687 & 0.6887 & 0.7027 & 0.4522
& 0.6698 & 0.5775 & 0.5958 & 0.4607
& -- \\
\midrule

\multirow{3}{*}{\shortstack{FU}}
& w/o Projection
& 0.7881 & 0.7702 & 0.6928 & 0.5622
& 0.8193 & 0.6348 & 0.6796 & 0.4171
& 0.5823 & 0.5566 & 0.5379 & 0.5268
& 30.8\% \\
& w/o Task
& 0.2599 & 0.2943 & 0.2627 & 0.5509
& 0.4530 & 0.4294 & 0.4018 & 0.4332
& 0.3019 & 0.2969 & 0.2917 & 0.4932
& 164.4\% \\
& Ours
& 0.8728 & 0.7543 & 0.7028 & 0.5197
& 0.8556 & 0.7716 & 0.7063 & 0.4868
& 0.6217 & 0.5696 & 0.5444 & 0.4918
& \textbf{22.0\%} \\
\midrule

\multirow{5}{*}{\shortstack{PU\\(SEG)}}
& w/o Projection
& 0.8728 & 0.7737 & 0.7181 & 0.5580
& 0.8241 & 0.8532 & 0.6885 & 0.6367
& 0.6030 & 0.6080 & 0.5153 & 0.5977
& 20.5\% \\
& w/o Task
& 0.2804 & 0.3164 & 0.2800 & 0.5464
& 0.8638 & 0.8652 & 0.7188 & 0.6055
& 0.6285 & 0.6176 & 0.5574 & 0.5551
& 76.6\% \\
& w/o Inst
& 0.8406 & 0.7469 & 0.6755 & 0.5520
& 0.8597 & 0.8412 & 0.7011 & 0.5891
& 0.6325 & 0.6163 & 0.5500 & 0.5547
& 22.1\% \\
& w/o Clean
& 0.8296 & 0.7271 & 0.6627 & 0.5290
& 0.8357 & 0.8692 & 0.7006 & 0.6486
& 0.6151 & 0.6305 & 0.5558 & 0.5723
& 23.6\% \\
& Ours
& 0.8707 & 0.7972 & 0.7073 & 0.5545
& 0.8499 & 0.8695 & 0.7085 & 0.6300
& 0.6316 & 0.6315 & 0.5562 & 0.5769
& \textbf{16.0\%} \\
\midrule

\multirow{5}{*}{\shortstack{PU\\(DEP)}}
& w/o Projection
& 0.9310 & 0.9382 & 0.7535 & 0.9770
& 0.8100 & 0.6589 & 0.6771 & 0.4310
& 0.6076 & 0.6066 & 0.5167 & 0.5941
& 19.5\% \\
& w/o Task
& 0.9344 & 0.9323 & 0.7558 & 0.9677
& 0.5915 & 0.5702 & 0.5262 & 0.4735
& 0.6322 & 0.6187 & 0.5572 & 0.5555
& 37.0\% \\
& w/o Inst
& 0.9217 & 0.9205 & 0.7532 & 0.9556
& 0.8288 & 0.7031 & 0.6967 & 0.4410
& 0.6180 & 0.5672 & 0.5476 & 0.5456
& 19.8\% \\
& w/o Clean
& 0.9278 & 0.9403 & 0.7480 & 0.9785
& 0.8134 & 0.7024 & 0.6913 & 0.4486
& 0.6125 & 0.6289 & 0.5475 & 0.5669
& 15.5\% \\
& Ours
& 0.9323 & 0.9410 & 0.7541 & 0.9787
& 0.8337 & 0.7035 & 0.6929 & 0.4411
& 0.6321 & 0.6313 & 0.5556 & 0.5780
& \textbf{12.7\%} \\
\midrule

\multirow{5}{*}{\shortstack{PU\\(NOR)}}
& w/o Projection
& 0.9266 & 0.9477 & 0.7527 & 0.9863
& 0.8248 & 0.8788 & 0.6964 & 0.6742
& 0.6023 & 0.5331 & 0.5092 & 0.5052
& 20.7\% \\
& w/o Task
& 0.9295 & 0.9369 & 0.7502 & 0.9736
& 0.8596 & 0.8641 & 0.7034 & 0.6164
& 0.2841 & 0.2716 & 0.2847 & 0.4966
& 58.7\% \\
& w/o Inst
& 0.9213 & 0.9160 & 0.7504 & 0.9665
& 0.8550 & 0.8474 & 0.7068 & 0.5894
& 0.6164 & 0.5394 & 0.5430 & 0.4699
& 12.9\% \\
& w/o Clean
& 0.9295 & 0.9398 & 0.7517 & 0.9823
& 0.8391 & 0.8755 & 0.6976 & 0.6566
& 0.6039 & 0.5284 & 0.5365 & 0.4619
& 14.1\% \\
& Ours
& 0.9305 & 0.9463 & 0.7531 & 0.9858
& 0.8537 & 0.8757 & 0.7054 & 0.6573
& 0.6192 & 0.5486 & 0.5425 & 0.4783
& \textbf{12.4\%} \\
\bottomrule
\end{tabular}
}
\end{table*}

Table~\ref{tab:nyuv2_ablation} analyzes each component on NYUv2. We compare the full method with four variants: \textbf{(i) w/o Projection} removes task-aware gradient projection; \textbf{(ii) w/o Task} removes same-task constraints across different instances; \textbf{(iii) w/o Inst} removes cross-task constraints on the same instance; and \textbf{(iv) w/o Clean} removes constraints from clean retain samples. Under FU, w/o Task causes the largest degradation, increasing UIS from $22.0\%$ to $164.4\%$, showing that same-task retain constraints are crucial for protecting retain instances. Removing Projection also increases UIS to $30.8\%$, confirming the role of task-aware projection. For PU, w/o Task remains the most influential ablation, especially for SEG under PU(SEG) and NOR under PU(NOR), where UIS rises to $76.6\%$ and $58.7\%$, respectively. In contrast, w/o Inst and w/o Clean cause milder increases. Overall, the full method best mitigates both task- and instance-level interference.

\subsection{Scalability Analysis}

Figure~\ref{fig:unlearn_ratio} further evaluates scalability under increasing unlearn ratios on NYUv2, from $10\%$ to $50\%$. Our method remains stable across all settings, keeping UIS around $20\%$--$25\%$ even when half of the data is unlearned. In contrast, the baselines degrade rapidly as the unlearn ratio increases. SCRUB is already about twice as large as ours at low ratios and becomes unstable at $50\%$, with UIS increasing to over $200\%$ in FU, PU(SEG), and PU(DEP), and to over $120\%$ in PU(NOR). NegGrad+ also rises substantially, often exceeding $50\%$ under larger unlearn ratios. These failures are mainly caused by large drops in \textsc{Ret} and \textsc{Val}, indicating that the model loses retained utility and generalization rather than only forgetting the target data. These results show that our method can handle large-scale multi-task unlearning in a stable manner, effectively removing increasing amounts of supervision while avoiding severe damage to model performance.
\begin{figure}[t]
    \centering
    \begin{subfigure}[t]{0.24\linewidth}
        \centering
        \includegraphics[width=\linewidth]{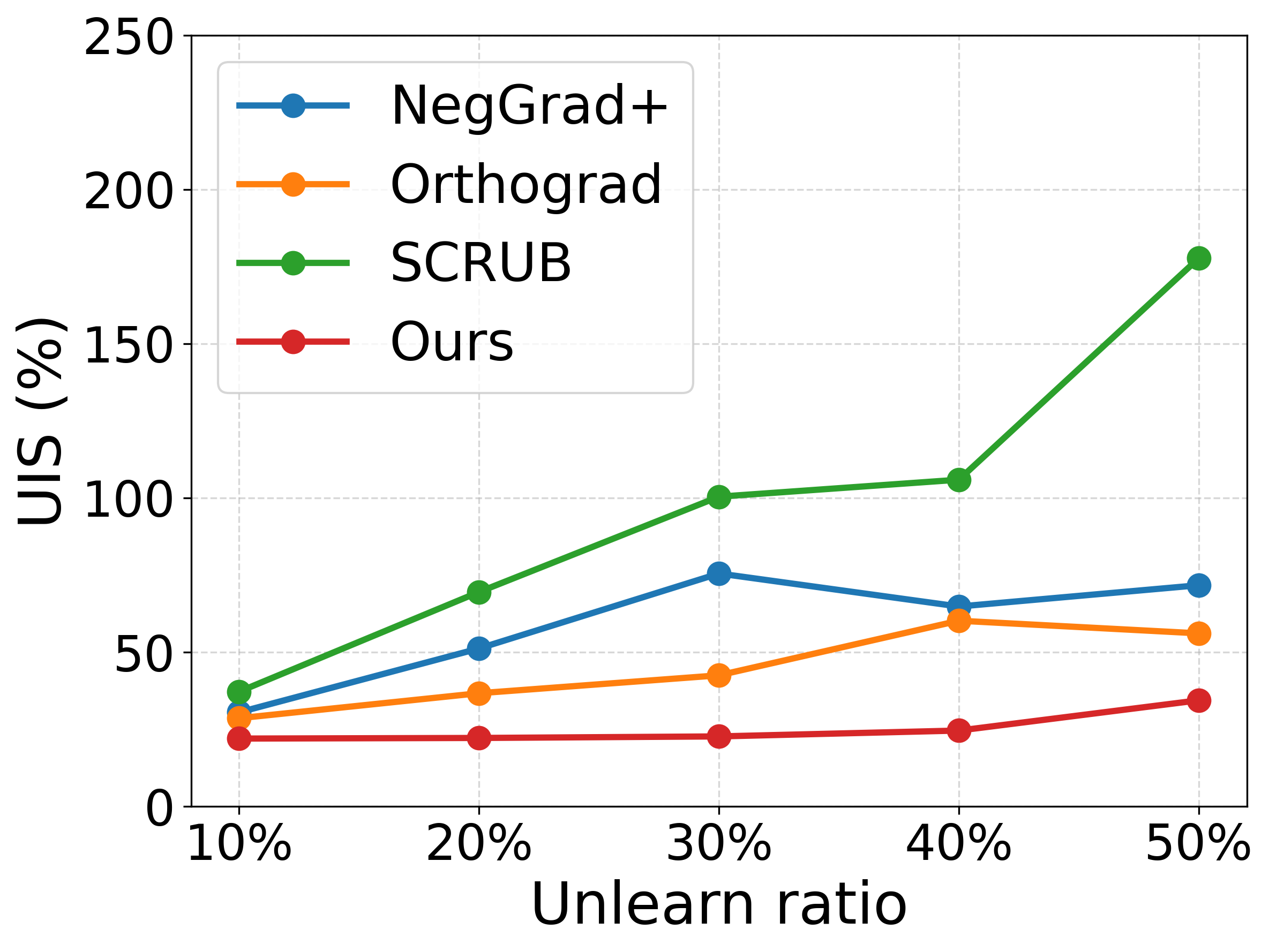}
        \caption{FU}
        \label{fig:ratio_fu}
    \end{subfigure}
    \hfill
    \begin{subfigure}[t]{0.24\linewidth}
        \centering
        \includegraphics[width=\linewidth]{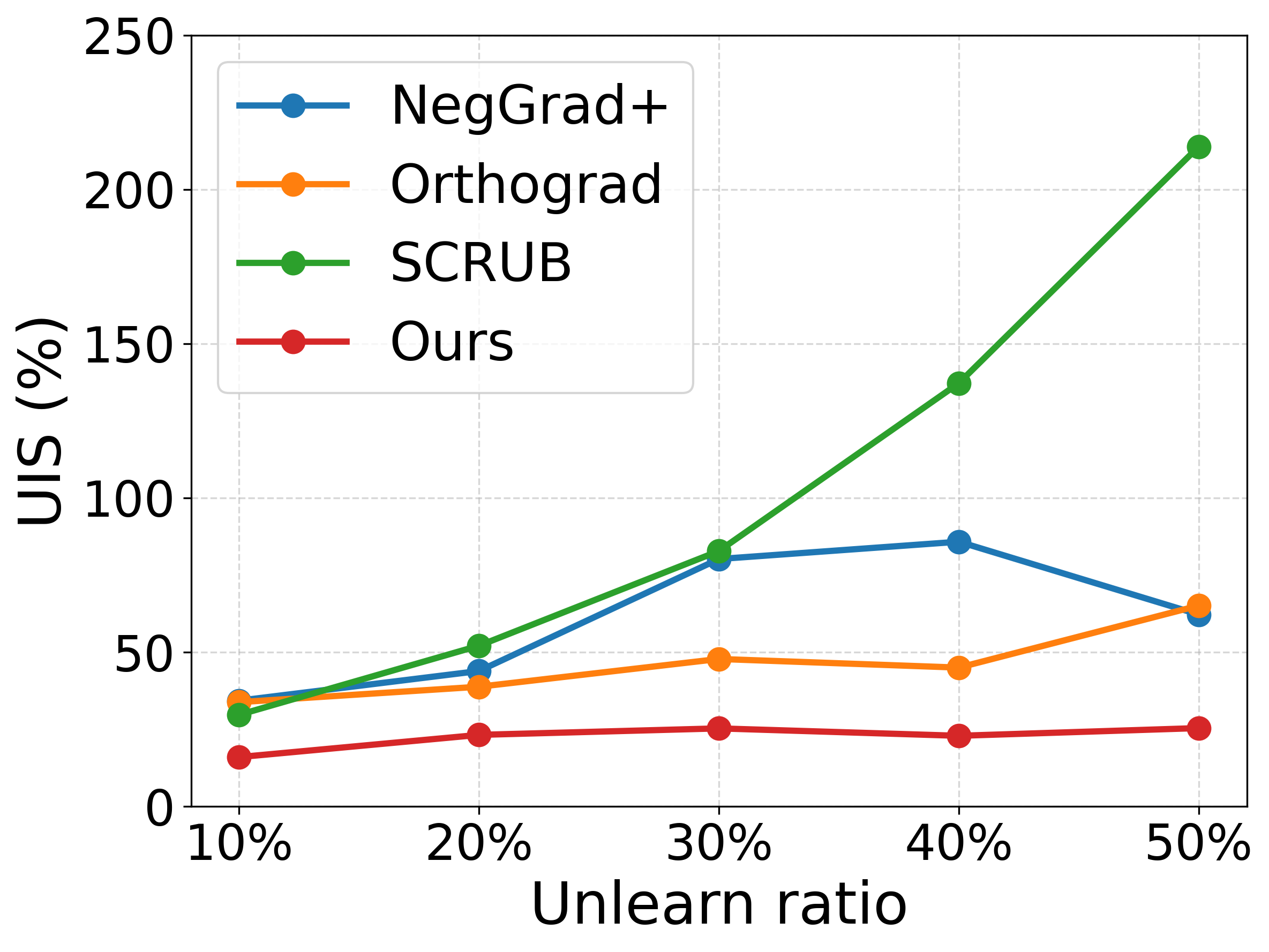}
        \caption{PU (SEG)}
        \label{fig:ratio_seg}
    \end{subfigure}
    \hfill
    \begin{subfigure}[t]{0.24\linewidth}
        \centering
        \includegraphics[width=\linewidth]{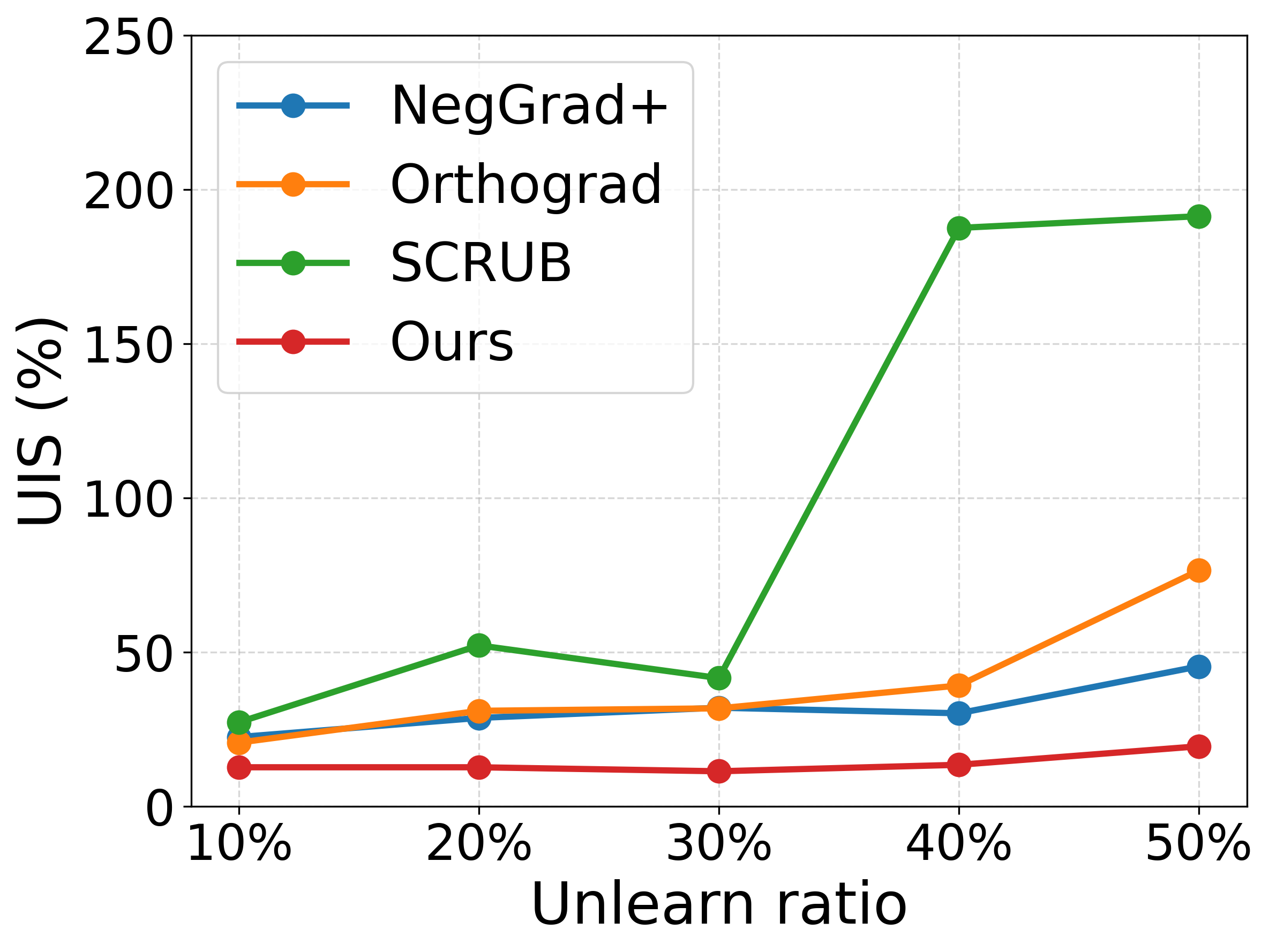}
        \caption{PU (DEP)}
        \label{fig:ratio_dep}
    \end{subfigure}
    \hfill
    \begin{subfigure}[t]{0.24\linewidth}
        \centering
        \includegraphics[width=\linewidth]{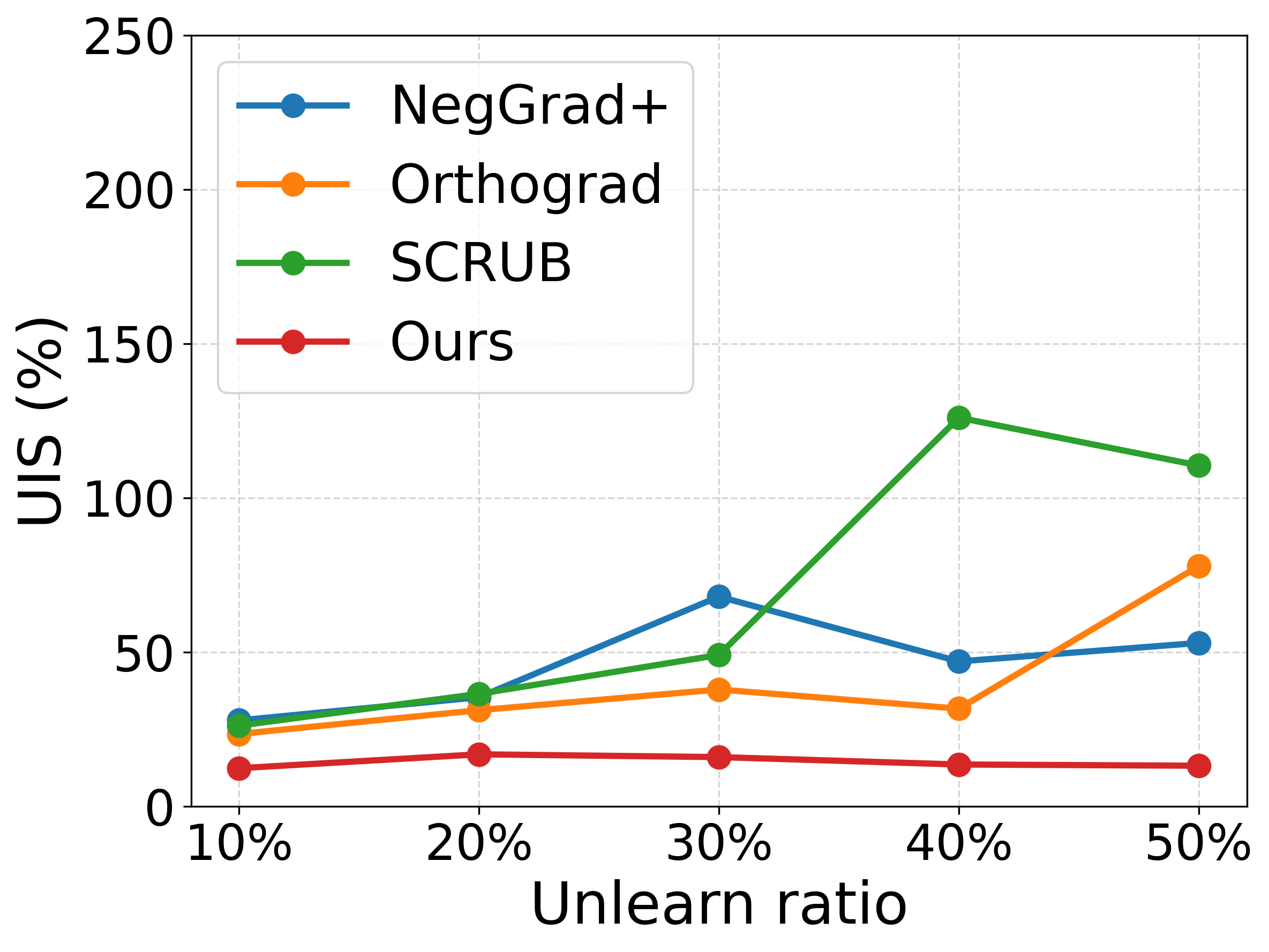}
        \caption{PU (NOR)}
        \label{fig:ratio_nor}
    \end{subfigure}

    \caption{
    Unlearning impact score (\%) under different unlearn ratios. Lower is better}
    \label{fig:unlearn_ratio}
\end{figure}
\section{Related Work}
Machine unlearning~\cite{cao2015towards,bourtoule2021machine,graves2021amnesiac,guo2019certified,golatkar2020eternal,ding2024unified,cha2024towards} aims to remove sensitive or outdated information from trained models while preserving utility. It is important for privacy~\cite{voigt2017eu}, security~\cite{huang2025survey}, fairness~\cite{zhang2024forgotten}, and robustness~\cite{qian2023towards}, and supports applications such as debiasing~\cite{chen2023fast}, debugging~\cite{surve2025explaining}, and auditing~\cite{wang2025tape}. Early work studied exact unlearning~\cite{cao2015towards,bourtoule2021machine,dukler2023safe,yan2022arcane,chowdhury2025towards}, which partitions data or computation so only affected subsets are retrained, but remains costly with repeated training. Later work explored approximate unlearning~\cite{graves2021amnesiac,golatkar2020eternal,roy2025novo,tong2025lethevit,guo2019certified,zhao2024makes,kurmanji2023towards}, which updates parameters to approximate retraining behavior through negative-gradient updates~\cite{kurmanji2023towards}, noise-based supervision~\cite{graves2021amnesiac}, or influence- and Fisher-based targeted updates~\cite{guo2019certified,golatkar2020eternal}. As models scale~\cite{liu2021swin,dosovitskiy2021an}, even fine-tuning-based unlearning becomes impractical~\cite{roy2025novo,tong2025lethevit}, since it must forget target data while preserving retain-set performance. Recent work therefore adopts parameter-efficient fine-tuning and treats unlearning as model editing~\cite{liu2025rethinking}, where lightweight adapters learn the unlearning behavior and are merged into the base model~\cite{ding2024unified,cha2024towards,liu2025lune}. However, most methods focus on single-task unlearning~\cite{shamsian2025go, foster2024fast, spartalis2025lotus, khalil2025coun, ebrahimpour2025amun}, whereas modern models often operate in multi-task settings~\cite{kamalesh2024unolora,xin2024vmt,agiza2024mtlora,yadav2023ties,yu2024language,prabhakar2025lora} with shared backbones. Because tasks are coupled through shared parameters, forgetting one task can unintentionally affect others, making multi-task unlearning substantially more challenging.

\section{Limitations and Conclusion}

In this paper, we introduced multi-task unlearning under two complementary scenarios: full-task unlearning, which removes all task-specific supervision associated with target instances, and partial-task unlearning, which removes only selected task supervision while preserving remaining valid predictions. We theoretically showed that multi-task unlearning inherently induces task- and instance-level interference. To address this issue, we proposed a framework that combines task-aware gradient projection with instance-level gradient orthogonalization. The former localizes updates to task-specific subspaces, while the latter removes forget-gradient components that conflict with retained supervision. Experiments on NYUv2 and Pascal show that our method better matches the desired reference behavior in both full- and partial-task unlearning. A limitation of our study is that the evaluation focuses on vision benchmarks. Future work will extend multi-task unlearning to text and audio domains, where task definitions and supervision structures may differ substantially.

\bibliographystyle{abbrv}
\bibliography{reference}

\clearpage
\appendix
\section{Notations}
\label{apx:notation}
\begin{table}[ht]
\centering
\small
\caption{Summary of notations.}
\label{tab:notation}
\begin{tabular}{p{0.15\linewidth}p{0.83\linewidth}}
\toprule
\textbf{Notation} & \textbf{Description} \\
\midrule
$\mathcal{X}$ & Set of input instances. \\
$\mathbf{x}_i$ & The $i$-th input instance. \\
$\mathcal{T}$ & Set of tasks. \\
$t$ & Task index. \\
$y_i^{(t)}$ & Supervision signal of instance $\mathbf{x}_i$ for task $t$. \\
$\mathcal{D}$ & Multi-task dataset containing task-instance supervision pairs. \\
$\theta$ & Shared model parameters. \\
$\theta^\star$ & Model trained on the full training set. \\
$\theta_r$ & Retrained model using only the retain set. \\
$f_t(\cdot;\theta)$ & Predictor for task $t$. \\
$\ell_{i,t}(\theta)$ & Loss of instance $\mathbf{x}_i$ on task $t$. \\
$L_r(\theta)$ & Empirical loss on the retain set. \\
$L_f(\theta)$ & Empirical loss on the forget set. \\
\midrule
$\mathcal{X}_f$ & Set of instances to be forgotten. \\
$\mathcal{X}_r$ & Retained instance set, $\mathcal{X}\setminus\mathcal{X}_f$. \\
$\mathcal{T}_f$ & Set of tasks whose supervision should be removed. \\
$\mathcal{T}_r$ & Retained task set, $\mathcal{T}\setminus\mathcal{T}_f$. \\
$\mathcal{D}_f$ & Forget set, containing supervision on forgotten instances for forgotten tasks. \\
$\mathcal{D}_r$ & Retain set, $\mathcal{D}\setminus\mathcal{D}_f$. \\
$\mathcal{D}_r^{\mathrm{task}}$ & Retained supervision on forgotten instances for retained tasks. \\
$\mathcal{D}_r^{\mathrm{inst}}$ & Retained supervision on retained instances for forgotten tasks. \\
$\mathcal{D}_r^{\mathrm{clean}}$ & Retained supervision on retained instances for retained tasks. \\
\midrule
$\mathbf{W}^\star$ & Frozen pretrained weight matrix. \\
$\widetilde{\mathbf{W}}$ & Edited weight matrix after low-rank unlearning. \\
$\mathbf{A},\mathbf{B}$ & Learnable low-rank factors. \\
$r$ & Shared low-rank dimension. \\
$s$ & Task-specific subspace dimension. \\
$\mathbf{U}_t$ & Orthonormal basis for the task-specific subspace of task $t$. \\
$\mathbf{P}_t$ & Task-specific projection matrix, $\mathbf{P}_t=\mathbf{U}_t\mathbf{U}_t^\top$. \\
$\nabla_{\mathbf{A}}^{(t)},\nabla_{\mathbf{B}}^{(t)}$ & Projected task-specific gradients for $\mathbf{A}$ and $\mathbf{B}$. \\
$\nabla_{\mathbf{A},f},\nabla_{\mathbf{B},f}$ & Forget gradients for $\mathbf{A}$ and $\mathbf{B}$. \\
$\nabla_{\mathbf{A},r},\nabla_{\mathbf{B},r}$ & Retain gradients for $\mathbf{A}$ and $\mathbf{B}$. \\
$\nabla_{\mathbf{A},f}^{\perp},\nabla_{\mathbf{B},f}^{\perp}$ & Retain-aware orthogonalized forget gradients. \\
$\eta_1,\eta_2$ & Step sizes controlling retention and forgetting strengths. \\
\midrule
$\mathbf{H}_r$ & Hessian of the retain loss at $\theta_r$. \\
$\rho$ & Ratio between forget and retain set sizes, $|\mathcal{D}_f|/|\mathcal{D}_r|$. \\
$\gamma$ & Desired first-order unlearning effect. \\
$\delta$ & Parameter update direction. \\
\bottomrule
\end{tabular}
\end{table}
\section{Detailed Proofs}

\subsection{Proof of Theorem~\ref{thm:interference}}

\begin{customthm}{\ref{thm:interference}}
Assume that $L_r$ and $L_f$ are twice differentiable, $\mathbf{H}_r := \nabla^2 L_r(\theta_r)$ is invertible, and $|\mathcal{D}_r| \gg |\mathcal{D}_f|$. Let $\rho := |\mathcal{D}_f|/|\mathcal{D}_r|$, and suppose that $\theta^\star$ locally minimizes $L_r(\theta)+\rho L_f(\theta)$. Then, removing $\mathcal{D}_f$ induces the following loss change for any retained task-instance pair $(\mathbf{x}_i,t,y_i^{(t)})\in\mathcal{D}_r$:
\begingroup\small
\begin{equation*}
\ell_{i,t}(\theta_r)-\ell_{i,t}(\theta^\star)
=
\underbrace{
\rho\,
\nabla \ell_{i,t}(\theta_r)^\top
\mathbf{H}_r^{-1}
\nabla L_f(\theta_r)
}_{\text{first-order}}
+
\underbrace{
O(\rho^2)
}_{\text{higher-order}}.
\end{equation*}
\endgroup
\end{customthm}

\begin{proof}
By the definitions of $\theta^\star$ and $\theta_r$ as local optima, we have
\[
\nabla L(\theta^\star)=0,
\qquad
\nabla L_r(\theta_r)=0,
\]
where
\[
L(\theta)
=
\frac{|\mathcal{D}_r|}{|\mathcal{D}|}L_r(\theta)
+
\frac{|\mathcal{D}_f|}{|\mathcal{D}|}L_f(\theta).
\]
Multiplying $\nabla L(\theta^\star)=0$ by $|\mathcal{D}|/|\mathcal{D}_r|$ gives
\[
\nabla L_r(\theta^\star)+\rho\,\nabla L_f(\theta^\star)=0,
\]
where $\rho=|\mathcal{D}_f|/|\mathcal{D}_r|$.

Applying Taylor expansion around $\theta_r$, we obtain
\[
\nabla L_r(\theta^\star)
=
\mathbf{H}_r(\theta^\star-\theta_r)
+
O(\|\theta^\star-\theta_r\|^2),
\]
where we use $\nabla L_r(\theta_r)=0$ and 
$\mathbf{H}_r=\nabla^2 L_r(\theta_r)$. Similarly,
\[
\nabla L_f(\theta^\star)
=
\nabla L_f(\theta_r)
+
O(\|\theta^\star-\theta_r\|).
\]

Substituting these expansions into 
\[
\nabla L_r(\theta^\star)+\rho\,\nabla L_f(\theta^\star)=0
\]
yields
\[
\mathbf{H}_r(\theta^\star-\theta_r)
+
\rho\,\nabla L_f(\theta_r)
=
O(\|\theta^\star-\theta_r\|^2)
+
O(\rho\|\theta^\star-\theta_r\|).
\]
Since $\mathbf{H}_r$ is invertible and $\rho\ll 1$, it follows that
\[
\theta_r-\theta^\star
=
\rho\,\mathbf{H}_r^{-1}\nabla L_f(\theta_r)
+
O(\rho^2).
\]

Applying Taylor expansion of $\ell_{i,t}$ around $\theta_r$ gives
\[
\ell_{i,t}(\theta^\star)
=
\ell_{i,t}(\theta_r)
+
\nabla \ell_{i,t}(\theta_r)^\top(\theta^\star-\theta_r)
+
O(\|\theta^\star-\theta_r\|^2).
\]
Rearranging and substituting the above result with 
$\|\theta^\star-\theta_r\|=O(\rho)$ yields
\[
\ell_{i,t}(\theta_r)-\ell_{i,t}(\theta^\star)
=
\rho\,
\nabla \ell_{i,t}(\theta_r)^\top
\mathbf{H}_r^{-1}
\nabla L_f(\theta_r)
+
O(\rho^2).
\]
The theorem follows.
\end{proof}

\subsection{Proof of Corollary~\ref{cor:interference}}

\begin{customcor}{\ref{cor:interference}}
Task-level and instance-level interference correspond to aggregations over $\mathcal{D}_r^{\mathrm{task}}$ and $\mathcal{D}_r^{\mathrm{inst}}$, respectively, and are both governed by the same first-order term in Theorem~\ref{thm:interference}.
\end{customcor}

\begin{proof}
From Theorem~\ref{thm:interference}, for any retained task-instance pair
$(\mathbf{x}_i,t,y_i^{(t)})\in\mathcal{D}_r$, the loss change induced by removing $\mathcal{D}_f$ satisfies
\begin{equation}
\ell_{i,t}(\theta_r)-\ell_{i,t}(\theta^\star)
=
\rho\,
\nabla \ell_{i,t}(\theta_r)^\top
\mathbf{H}_r^{-1}
\nabla L_f(\theta_r)
+ O(\rho^2),
\end{equation}
where $\rho=|\mathcal{D}_f|/|\mathcal{D}_r|$.

Aggregating over any subset $\mathcal{S}\subseteq\mathcal{D}_r$ yields
\begin{equation}
\sum_{(\mathbf{x}_i,t,y_i^{(t)})\in\mathcal{S}}
\bigl(\ell_{i,t}(\theta_r)-\ell_{i,t}(\theta^\star)\bigr)
=
\rho
\sum_{(\mathbf{x}_i,t,y_i^{(t)})\in\mathcal{S}}
\nabla \ell_{i,t}(\theta_r)^\top
\mathbf{H}_r^{-1}
\nabla L_f(\theta_r)
+ O(|\mathcal{S}|\rho^2).
\end{equation}

Applying this result to $\mathcal{D}_r^{\mathrm{task}}$ and $\mathcal{D}_r^{\mathrm{inst}}$ gives
\begin{align}
\sum_{(\mathbf{x}_i,t,y_i^{(t)})\in\mathcal{D}_r^{\mathrm{task}}}
\bigl(\ell_{i,t}(\theta_r)-\ell_{i,t}(\theta^\star)\bigr)
&=
\rho
\sum_{(\mathbf{x}_i,t,y_i^{(t)})\in\mathcal{D}_r^{\mathrm{task}}}
\nabla \ell_{i,t}(\theta_r)^\top
\mathbf{H}_r^{-1}
\nabla L_f(\theta_r)
+ O\!\left(|\mathcal{D}_r^{\mathrm{task}}|\rho^2\right), \\
\sum_{(\mathbf{x}_i,t,y_i^{(t)})\in\mathcal{D}_r^{\mathrm{inst}}}
\bigl(\ell_{i,t}(\theta_r)-\ell_{i,t}(\theta^\star)\bigr)
&=
\rho
\sum_{(\mathbf{x}_i,t,y_i^{(t)})\in\mathcal{D}_r^{\mathrm{inst}}}
\nabla \ell_{i,t}(\theta_r)^\top
\mathbf{H}_r^{-1}
\nabla L_f(\theta_r)
+ O\!\left(|\mathcal{D}_r^{\mathrm{inst}}|\rho^2\right).
\end{align}

Therefore, task-level and instance-level interference are governed by the same first-order Hessian-preconditioned gradient coupling term, differing only in the aggregation domain. The corollary follows.
\end{proof}

\subsection{Proof of Proposition~\ref{prop:suboptimal}}

\begin{customprop}{\ref{prop:suboptimal}}
Under the local quadratic approximation
$L_r(\theta_r+\delta)\approx L_r(\theta_r)+\frac{1}{2}\delta^\top \mathbf{H}_r\delta$, assume that $\mathbf{H}_r$ is positive definite. Consider updates $\delta$ that achieve a fixed first-order unlearning effect, $\nabla L_f(\theta_r)^\top \delta = \gamma$, for some desired level $\gamma>0$. Then, the unique retain-loss-minimizing update is
\begingroup\small\begin{equation*}
\delta^\star
=
\frac{\gamma}
{\nabla L_f(\theta_r)^\top \mathbf{H}_r^{-1}\nabla L_f(\theta_r)}
\mathbf{H}_r^{-1}\nabla L_f(\theta_r).
\end{equation*}\endgroup
Hence, directly following the unlearning gradient, i.e., $\delta=\alpha \nabla L_f(\theta_r)$, is generally suboptimal unless $\nabla L_f(\theta_r)$ is an eigenvector of $\mathbf{H}_r$.
\end{customprop}

\begin{proof}
Under the local quadratic approximation,
\[
L_r(\theta_r+\delta)
\approx
L_r(\theta_r)+\frac{1}{2}\delta^\top \mathbf{H}_r\delta.
\]
Since $L_r(\theta_r)$ is constant with respect to $\delta$, minimizing the retain loss is equivalent to solving
\[
\min_{\delta}\ \frac{1}{2}\delta^\top \mathbf{H}_r\delta
\qquad
\text{s.t.}
\qquad
\nabla L_f(\theta_r)^\top \delta = \gamma.
\]
Let $g:=\nabla L_f(\theta_r)$. The problem becomes
\[
\min_{\delta}\ \frac{1}{2}\delta^\top \mathbf{H}_r\delta
\qquad
\text{s.t.}
\qquad
g^\top \delta = \gamma.
\]

We form the Lagrangian
\[
\mathcal{L}(\delta,\lambda)
=
\frac{1}{2}\delta^\top \mathbf{H}_r\delta
-
\lambda(g^\top \delta-\gamma).
\]
Taking the derivative with respect to $\delta$ and setting it to zero gives
\[
\nabla_\delta \mathcal{L}(\delta,\lambda)
=
\mathbf{H}_r\delta-\lambda g
=
0.
\]
Hence,
\[
\delta=\lambda \mathbf{H}_r^{-1}g.
\]
Substituting this into the constraint $g^\top \delta=\gamma$ yields
\[
g^\top(\lambda \mathbf{H}_r^{-1}g)=\gamma,
\]
and therefore
\[
\lambda
=
\frac{\gamma}{g^\top \mathbf{H}_r^{-1}g}.
\]
Thus,
\[
\delta^\star
=
\frac{\gamma}{g^\top \mathbf{H}_r^{-1}g}\mathbf{H}_r^{-1}g
=
\frac{\gamma}
{\nabla L_f(\theta_r)^\top \mathbf{H}_r^{-1}\nabla L_f(\theta_r)}
\mathbf{H}_r^{-1}\nabla L_f(\theta_r).
\]
Since $\mathbf{H}_r\succ 0$, the objective is strictly convex and the solution is unique.

Finally, consider directly following the unlearning gradient, i.e.,
\[
\delta=\alpha g
\]
for some scalar $\alpha$. For this update to coincide with the optimum $\delta^\star$, we must have
\[
\alpha g
\propto
\mathbf{H}_r^{-1}g,
\]
which holds only if
\[
\mathbf{H}_r^{-1}g=\beta g
\]
for some scalar $\beta$, or equivalently,
\[
\mathbf{H}_r g=\frac{1}{\beta}g.
\]
Thus, $g=\nabla L_f(\theta_r)$ must be an eigenvector of $\mathbf{H}_r$. Therefore, directly following the unlearning gradient is generally suboptimal unless $\nabla L_f(\theta_r)$ is an eigenvector of $\mathbf{H}_r$. The proposition follows.
\end{proof}

\subsection{Proof of Theorem~\ref{thm:task_projection}}
\label{proof:task_projection}

\begin{theorem}[Task-level interference reduction]
\label{thm:task_projection}
For any $\mathbf{Z}\in\{\mathbf{A},\mathbf{B}\}$ and any two tasks $t$ and $t'$, let $\nabla_{\mathbf{Z},t}$ and $\nabla_{\mathbf{Z},t'}$ be their task gradients. The task-aware projected gradients defined in Eq.~\eqref{eq:projected_gradients} are $\nabla_{\mathbf{Z},t}^{(t)}=\nabla_{\mathbf{Z},t}\mathbf{P}_{t}$ and $\nabla_{\mathbf{Z},t'}^{(t')}=\nabla_{\mathbf{Z},t'}\mathbf{P}_{t'}$. If $\|\mathbf{U}_{t}^{\top}\mathbf{U}_{t'}\|_2\le \gamma_{t,t'}$, then
\begingroup\small\begin{equation}
\left|
\left\langle
\nabla_{\mathbf{Z},t}^{(t)},
\nabla_{\mathbf{Z},t'}^{(t')}
\right\rangle_F
\right|
\le
\gamma_{t,t'}
\|\nabla_{\mathbf{Z},t}\|_F
\|\nabla_{\mathbf{Z},t'}\|_F,
\end{equation}\endgroup
where $\langle\cdot,\cdot\rangle_F$ and $\|\cdot\|_F$ denote the Frobenius inner product and norm, respectively.
\end{theorem}
\begin{proof}
By definition,
$\nabla_{\mathbf{Z},t}^{(t)}=\nabla_{\mathbf{Z},t}\mathbf{P}_{t}$ and
$\nabla_{\mathbf{Z},t'}^{(t')}=\nabla_{\mathbf{Z},t'}\mathbf{P}_{t'}$. Since $\mathbf{P}_{t}$ and $\mathbf{P}_{t'}$ are symmetric projectors,
\begingroup\small\begin{equation}
\left|
\left\langle
\nabla_{\mathbf{Z},t}^{(t)},
\nabla_{\mathbf{Z},t'}^{(t')}
\right\rangle_F
\right|
=
\left|
\mathrm{tr}\!\left(
\mathbf{P}_{t}
\nabla_{\mathbf{Z},t}^{\top}
\nabla_{\mathbf{Z},t'}
\mathbf{P}_{t'}
\right)
\right|.
\end{equation}\endgroup
By cyclicity of trace and the Frobenius Cauchy--Schwarz inequality,
\begingroup\small\begin{equation}
\left|
\mathrm{tr}\!\left(
\mathbf{P}_{t}
\nabla_{\mathbf{Z},t}^{\top}
\nabla_{\mathbf{Z},t'}
\mathbf{P}_{t'}
\right)
\right|
=
\left|
\left\langle
\nabla_{\mathbf{Z},t},
\nabla_{\mathbf{Z},t'}\mathbf{P}_{t'}\mathbf{P}_{t}
\right\rangle_F
\right|
\le
\|\nabla_{\mathbf{Z},t}\|_F
\|\nabla_{\mathbf{Z},t'}\|_F
\|\mathbf{P}_{t'}\mathbf{P}_{t}\|_2 .
\end{equation}\endgroup
Because $\mathbf{P}_{t}=\mathbf{U}_{t}\mathbf{U}_{t}^{\top}$ and
$\mathbf{P}_{t'}=\mathbf{U}_{t'}\mathbf{U}_{t'}^{\top}$ with orthonormal bases, we have
\begingroup\small\begin{equation}
\|\mathbf{P}_{t'}\mathbf{P}_{t}\|_2
\le
\|\mathbf{U}_{t'}^{\top}\mathbf{U}_{t}\|_2
=
\|\mathbf{U}_{t}^{\top}\mathbf{U}_{t'}\|_2
\le
\gamma_{t,t'}.
\end{equation}\endgroup
Combining the above inequalities proves the claim.
\end{proof}

\subsection{Proof of Theorem~\ref{thm:instance_orthogonalization}}
\label{proof:instance_orthogonalization}

\begin{theorem}[Instance-level interference reduction]
\label{thm:instance_orthogonalization}
For any $\mathbf{Z}\in\{\mathbf{A},\mathbf{B}\}$, let $\nabla_{\mathbf{Z},f}^{\perp}=\Pi_{\nabla_{\mathbf{Z},r}}^{\perp}(\nabla_{\mathbf{Z},f})$ be the orthogonalized forget gradient defined following Eq.~\eqref{eq:orthogonalization}. Then its alignment with the retain gradient satisfies
\begingroup\small\begin{equation}
\left|
\left\langle
\nabla_{\mathbf{Z},f}^{\perp},
\nabla_{\mathbf{Z},r}
\right\rangle_F
\right|
=
\frac{\varepsilon}{\|\nabla_{\mathbf{Z},r}\|_F^2+\varepsilon}
\left|
\left\langle
\nabla_{\mathbf{Z},f},
\nabla_{\mathbf{Z},r}
\right\rangle_F
\right| .
\end{equation}\endgroup
When $\varepsilon=0$, the orthogonalized forget gradient is exactly orthogonal to the retain gradient.
\end{theorem}
\begin{proof}
By the definition of the orthogonalized forget gradient,
\begingroup\small\begin{equation}
\nabla_{\mathbf{Z},f}^{\perp}
=
\nabla_{\mathbf{Z},f}
-
\frac{
\left\langle
\nabla_{\mathbf{Z},f},
\nabla_{\mathbf{Z},r}
\right\rangle_F
}{
\|\nabla_{\mathbf{Z},r}\|_F^2+\varepsilon
}
\nabla_{\mathbf{Z},r}.
\end{equation}\endgroup
Taking its Frobenius inner product with $\nabla_{\mathbf{Z},r}$ gives
\begingroup\small\begin{equation}
\begin{aligned}
\left\langle
\nabla_{\mathbf{Z},f}^{\perp},
\nabla_{\mathbf{Z},r}
\right\rangle_F
&=
\left\langle
\nabla_{\mathbf{Z},f},
\nabla_{\mathbf{Z},r}
\right\rangle_F
-
\frac{
\left\langle
\nabla_{\mathbf{Z},f},
\nabla_{\mathbf{Z},r}
\right\rangle_F
}{
\|\nabla_{\mathbf{Z},r}\|_F^2+\varepsilon
}
\left\langle
\nabla_{\mathbf{Z},r},
\nabla_{\mathbf{Z},r}
\right\rangle_F \\
&=
\frac{\varepsilon}{\|\nabla_{\mathbf{Z},r}\|_F^2+\varepsilon}
\left\langle
\nabla_{\mathbf{Z},f},
\nabla_{\mathbf{Z},r}
\right\rangle_F .
\end{aligned}
\end{equation}\endgroup
Taking absolute values yields the stated result. When $\varepsilon=0$, the right-hand side becomes zero, so $\nabla_{\mathbf{Z},f}^{\perp}$ is orthogonal to $\nabla_{\mathbf{Z},r}$.
\end{proof}

\section{Additional Experimental Results}

\subsection{Detailed Setup}

\subsubsection{Datasets}
\label{sec:exp_datasets}

\vparagraph{Benchmark.}
To evaluate multi-task unlearning under diverse supervision granularities, we use two representative vision benchmarks spanning \emph{image-level}, \emph{instance-level}, and \emph{pixel-level} prediction.

We first use NYUv2~\cite{silberman2012indoor} as a dense multi-task benchmark with pixel-level supervision. NYUv2 consists of indoor RGB-D images annotated with aligned pixel-wise labels, enabling multiple dense prediction tasks to be defined on the same input. In our setting, we consider three tasks: semantic segmentation (SEG), depth estimation (DEP), and surface-normal prediction (NOR), which share the same input images but differ in output structure and loss formulation. This benchmark allows us to study unlearning behavior under tightly coupled multi-task supervision.

We further use PASCAL VOC~\cite{Everingham10} to cover heterogeneous supervision across image-level and instance-level tasks. Specifically, we consider multi-label image classification (CLS) and object detection (OD), which differ in both annotation granularity and prediction format. Unlike NYUv2, these tasks involve partially shared supervision and different annotation scopes, providing a complementary setting to evaluate unlearning under loosely coupled tasks.

Together, NYUv2 and PASCAL VOC form a comprehensive testbed for multi-task unlearning, covering diverse task semantics, output structures, and supervision granularities, ranging from dense pixel-wise prediction to sparse instance- and image-level recognition.

\vparagraph{Dataset partitions.} We follow the official train/test splits of each benchmark~\cite{Everingham10,silberman2012indoor}. Within the training split, we further divide the data into disjoint retain (RET) and unlearn sets (UNL) to simulate realistic unlearning requests, where the unlearn set should be removed while the retain set should be preserved. The official test split (VAL) is kept held out for evaluation.

The \emph{retain set} contains supervision that should remain intact after unlearning, and is used to evaluate whether useful knowledge is preserved. The \emph{unlearn set} contains the target data to be forgotten, and is used to measure the effectiveness of forgetting. The \emph{validation set} is disjoint from both retain and unlearn sets, and is used to assess generalization beyond the directly affected training instances.

This partitioning enables a unified evaluation of the trade-offs between retention, forgetting, and generalization, which are inherently coupled in multi-task unlearning.

\subsubsection{Baselines}

We compare the proposed method with representative unlearning baselines from different methodological families:
\begin{itemize}[leftmargin=2em]
    \item \textbf{Original} denotes the model trained on the full training set, including both the retain and forget sets. It serves as the pre-unlearning reference, especially for supervision that should remain unchanged under partial-task unlearning.

    \item \textbf{Retrain} retrains the model from scratch using only the retain set. It serves as the ideal unlearning target and an upper-bound reference, albeit at substantially higher computational cost.

    \item \textbf{NegGrad+}~\cite{kurmanji2023towards,graves2021amnesiac} performs gradient ascent on the forget set while incorporating retain-side optimization to reduce damage to retained data.

    \item \textbf{Fisher}~\cite{golatkar2020eternal} estimates parameter importance via Fisher information and selectively perturbs parameters associated with the forget data.

    \item \textbf{Influence}~\cite{guo2019certified} applies influence-function-based updates to estimate and remove the effect of forget examples.

    \item \textbf{SSD}~\cite{foster2024fast} selectively dampens parameters strongly associated with the forget data without requiring full retraining.

    \item \textbf{OrthoGrad}~\cite{shamsian2025go} orthogonalizes forget updates against retained gradients to better preserve retained performance.

    \item \textbf{SCRUB}~\cite{kurmanji2023towards} combines forgetting and retention via a teacher-student framework.
\end{itemize}

These baselines cover major approximate unlearning strategies, including retraining-based references, first-order gradient updates, second-order methods based on influence functions or Fisher information, parameter-dampening approaches, and gradient-projection mechanisms.

\subsubsection{Backbones}
We evaluate our framework with two representative vision backbones: \textbf{ViT-L}~\cite{dosovitskiy2021an} and \textbf{Swin-L}~\cite{liu2021swin}. ViT-L follows the standard Vision Transformer architecture, which represents images as patch tokens and captures global dependencies through self-attention. It provides a strong fully attention-based backbone for evaluating unlearning under shared representations. Swin-L adopts a hierarchical Transformer design with shifted-window attention, enabling efficient local-to-global feature aggregation and strong performance on dense prediction tasks. Using both backbones allows us to examine whether the proposed unlearning framework remains effective across different Transformer architectures, from global patch-based attention to hierarchical window-based attention.

\subsubsection{Evaluation Metrics}

\vparagraph{Task utility metrics.}
We adopt standard evaluation metrics tailored to each task, covering pixel-level, geometric, and instance-level prediction quality.

\textit{Semantic Segmentation (mIoU).}
For semantic segmentation, we use mean Intersection-over-Union (mIoU)~\cite{long2015fully}. Let $C$ denote the number of classes, and let $\mathrm{TP}_c$, $\mathrm{FP}_c$, and $\mathrm{FN}_c$ denote the number of true positives, false positives, and false negatives for class $c$, respectively. The IoU for class $c$ is defined as
\begin{equation}
\mathrm{IoU}_c
=
\frac{\mathrm{TP}_c}{\mathrm{TP}_c + \mathrm{FP}_c + \mathrm{FN}_c},
\end{equation}
and the mean IoU is given by
\begin{equation}
\mathrm{mIoU}
=
\frac{1}{C} \sum_{c=1}^{C} \mathrm{IoU}_c.
\end{equation}

\textit{Depth Estimation (Threshold Accuracy).}
For depth estimation, we use Threshold Accuracy~\cite{silberman2012indoor}. Let $d_i$ and $\hat{d}_i$ denote the ground-truth and predicted depth at pixel $i$, respectively. For a threshold $\delta$, the accuracy is defined as
\begin{equation}
\mathrm{Acc}(\delta)
=
\frac{1}{N}
\sum_{i=1}^{N}
\mathbf{1}
\left(
\max\left(
\frac{\hat{d}_i}{d_i},
\frac{d_i}{\hat{d}_i}
\right)
< \delta
\right),
\end{equation}
where $N$ is the total number of pixels. Common choices include $\delta \in \{1.25, 1.25^2, 1.25^3\}$.

\textit{Surface Normal Prediction (Angular Accuracy).}
For surface-normal prediction, we measure Angular Accuracy~\cite{silberman2012indoor}. Let $\mathbf{n}_i$ and $\hat{\mathbf{n}}_i$ denote the ground-truth and predicted unit normal vectors at pixel $i$. The angular error is
\begin{equation}
\theta_i
=
\arccos\left(
\hat{\mathbf{n}}_i^\top \mathbf{n}_i
\right),
\end{equation}
and the accuracy under a threshold $\tau$ is defined as
\begin{equation}
\mathrm{Acc}(\tau)
=
\frac{1}{N}
\sum_{i=1}^{N}
\mathbf{1}(\theta_i < \tau),
\end{equation}
where $\tau$ is typically set to $11.25^\circ$, $22.5^\circ$, or $30^\circ$.

\textit{Object Detection and Classification (mAP).}
For object detection and multi-label classification, we use mean Average Precision (mAP)~\cite{lin2014microsoft,Everingham10}. For each class $c$, predictions are ranked by confidence scores, and precision–recall pairs are computed. The Average Precision (AP) for class $c$ is defined as the area under the precision–recall curve,
\begin{equation}
\mathrm{AP}_c
=
\int_{0}^{1} P_c(r)\, dr,
\end{equation}
where $P_c(r)$ denotes precision as a function of recall $r$. The mean Average Precision is then
\begin{equation}
\mathrm{mAP}
=
\frac{1}{C}
\sum_{c=1}^{C}
\mathrm{AP}_c.
\end{equation}
For object detection, AP is computed based on Intersection-over-Union (IoU) thresholds between predicted and ground-truth bounding boxes, following standard protocols.

\vparagraph{Privacy Metric: Membership Inference Attack.}
To assess residual memorization after unlearning, we evaluate task-wise Membership Inference
Attack (MIA)~\cite{shokri2017membership,yeom2018privacy,song2021systematic}. Given a trained model, an adversary aims to determine
whether a sample $(\mathbf{x}, y)$ was used during training. In our evaluation, we adopt a
loss-based MIA rather than training an additional attack classifier. Specifically, for each task
$t$, we compute the per-sample task loss $\ell^{(t)}(\mathbf{x}, y^{(t)})$ and use its negative
value as the membership score:
\[
s^{(t)}(\mathbf{x}) = -\ell^{(t)}(\mathbf{x}, y^{(t)}).
\]
This follows the intuition that member samples typically obtain lower losses and therefore higher
membership scores.

For each task, we evaluate two attacks: \textit{retain-vs-val}, which treats retain samples as
members and validation samples as non-members, and \textit{unlearn-vs-val}, which treats unlearn
samples as members and validation samples as non-members. We use \textit{unlearn-vs-val} as the
main privacy metric for measuring residual membership leakage from the forgotten data, while
\textit{retain-vs-val} is recorded to monitor the membership behavior of retained samples. We report
ROC-AUC as the MIA score and additionally record average precision (AP). Lower MIA AUC, or an AUC
closer to the retrained reference, indicates that the unlearned model leaks less membership
information about the forgotten samples.

\subsubsection{Unlearning Impact Score (UIS)}
While individual metrics capture task-specific performance, evaluating multi-task unlearning holistically is inherently challenging due to competing objectives, including retaining useful knowledge, forgetting target data, preserving generalization, and mitigating memorization. To provide a unified evaluation criterion, we follow prior work~\cite{tu2024towards, cheng2024remaining, ahmed2025towards, shamsian2025go} and assess each method against a desired reference behavior.

Specifically, forgotten tasks are expected to match the retrained model, while retained tasks should remain consistent with the original pre-unlearning model. Let $s_t^{*}$ denote the metric value of the evaluated unlearning model for task $t$, and let $\bar{s}_t^{*}$ denote the corresponding reference metric. For forgotten tasks $\mathcal{T}_f$, the reference is taken from the retrained model, whereas for retained tasks $\mathcal{T}_r$, the reference is taken from the original model.

To jointly capture performance across different evaluation dimensions, including retention, forgetting, generalization, and privacy, we define the task-level discrepancy as
\begin{equation}
s_t
=
\frac{|s^{\mathrm{retain}}_t - \bar{s}^{\mathrm{retain}}_t|}{\bar{s}^{\mathrm{retain}}_t}
+
\frac{|s^{\mathrm{unlearn}}_t - \bar{s}^{\mathrm{unlearn}}_t|}{\bar{s}^{\mathrm{unlearn}}_t}
+
\frac{|s^{\mathrm{val}}_t - \bar{s}^{\mathrm{val}}_t|}{\bar{s}^{\mathrm{val}}_t}
+
\frac{|s^{\mathrm{MIA}}_t - \bar{s}^{\mathrm{MIA}}_t|}{\bar{s}^{\mathrm{MIA}}_t}.
\end{equation}
This formulation normalizes deviations across heterogeneous metrics, ensuring comparability across tasks and evaluation criteria.

The overall discrepancy score is then defined as the average across all tasks,
\begin{equation}
S
=
\frac{1}{|\mathcal{T}|}
\sum_{t \in \mathcal{T}} s_t.
\end{equation}
A smaller value of $S$ indicates that the unlearned model better aligns with the desired behavior, achieving effective forgetting while preserving utility, generalization, and privacy.

Finally, the reference construction adapts to different unlearning settings. For full-task unlearning, where $\mathcal{T}_f = \mathcal{T}$ and $\mathcal{T}_r = \emptyset$, all tasks are compared against the retrained model. For partial-task unlearning, let $t_f$ denote the forgotten task, then $\mathcal{T}_f = \{t_f\}$ and $\mathcal{T}_r = \mathcal{T} \setminus \{t_f\}$. In this case, the forgotten task uses the retrained reference, while each retained task is compared to the original model.

\subsubsection{Implementation Details}
We instantiate our framework on two shared vision transformer backbones, ViT-L~\cite{dosovitskiy2021an} and Swin-L~\cite{liu2021swin}, both initialized from official HuggingFace checkpoints~\cite{wolf2019huggingface} pretrained on ImageNet~\cite{deng2009imagenet}. For NYUv2, we augment the backbone with an Atrous Spatial Pyramid Pooling (ASPP) module~\cite{chen2017rethinking} and attach task-specific heads for semantic segmentation, depth estimation, and surface-normal prediction. For PASCAL VOC, we use a linear head for image classification and a DETR-style detection head with 100 object queries~\cite{carion2020end}. To reflect a realistic shared-parameter multi-task setting, we use a single shared LoRA with rank 16~\cite{huang2023lorahub} across all tasks on the same dataset, making unlearning particularly challenging due to the coupling between forgetting updates and retained supervision in the shared adaptation space.

Following prior unlearning protocols~\cite{kurmanji2023towards,shamsian2025go}, we designate 10\% of the training instances as the unlearn set and use the remaining training data as the retain set. During unlearning, we further sample 10\% of the retain set as an anchor subset to preserve retain-set performance. All compared methods, including our proposed approach and all baselines, use exactly the same unlearn/retain split and anchor subset for a fair comparison. To reduce sampling bias, we repeat the random unlearning split 10 times and report the averaged results across runs, using the same sampled splits for all baselines. Before unlearning, we train the target model to be unlearned for 15 epochs on NYUv2 and 50 epochs on Pascal for all compared methods. Unless otherwise specified, unlearning is performed by optimizing the trainable low-rank update
parameters with AdamW~\cite{loshchilov2017decoupled} for up to 20 epochs. We apply early stopping based on the membership inference attack (MIA) score and select the
checkpoint whose MIA performance is closest to the retrained reference, indicating that the model can no longer reliably distinguish the
unlearn set from held-out validation data. We use a learning rate of $1\times10^{-4}$ with $\eta_1=1$ and $\eta_2=0.1$ for partial unlearning on segmentation, depth, normal, detection, and full-unlearning settings. and a learning rate of $3\times10^{-4}$ with $\eta_1=1$ and $\eta_2=0.1$ for partial unlearning on classification. We provide the accompanying source code with preprocessing, training, and evaluation to facilitate reproducibility. All experiments were conducted on a workstation equipped with a single NVIDIA GeForce RTX 4090 GPU with 24GB of memory.
\subsection{Additional Experiments on Swin-L}
\label{apx:swin}

\begin{table*}[t]
\centering
\small
\caption{
Quantitative results of multi-task unlearning on NYUv2 using Swin-L}
\label{tab:nyu_unlearning_swin}

\resizebox{\textwidth}{!}{
\begin{tabular}{llccccccccccccc}
\toprule
\multirow{3}{*}{Setting} & \multirow{3}{*}{Method}
& \multicolumn{4}{c}{\textbf{SEG}}
& \multicolumn{4}{c}{\textbf{DEP}}
& \multicolumn{4}{c}{\textbf{NOR}}
& \multirow{3}{*}{\textbf{UIS}$\downarrow$} \\
\cmidrule(lr){3-6} \cmidrule(lr){7-10} \cmidrule(lr){11-14}
&
& \textsc{Ret} & \textsc{Unl} & \textsc{Val} & \textsc{MIA}
& \textsc{Ret} & \textsc{Unl} & \textsc{Val} & \textsc{MIA}
& \textsc{Ret} & \textsc{Unl} & \textsc{Val} & \textsc{MIA}
& \\
&
& (mIoU) & (mIoU) & (mIoU) & (AUC)
& ($\sigma_{1.25}$) & ($\sigma_{1.25}$) & ($\sigma_{1.25}$) & (AUC)
& (A$_{30}$) & (A$_{30}$) & (A$_{30}$) & (AUC)
& \\
\midrule

\multirow{2}{*}{\shortstack{Ref.}}
& Original
& 0.8813 & 0.8759 & 0.7330 & 0.8918
& 0.8428 & 0.8331 & 0.7291 & 0.6067
& 0.5803 & 0.5776 & 0.5391 & 0.5519
& -- \\
& Retrain
& 0.8630 & 0.7312 & 0.7293 & 0.5295
& 0.8387 & 0.6854 & 0.7127 & 0.4283
& 0.5915 & 0.5242 & 0.5425 & 0.4239
& -- \\
\midrule

\multirow{7}{*}{\shortstack{FU}}
& NegGrad+
& 0.8321 & 0.7205 & 0.6819 & 0.5363
& 0.8011 & 0.6953 & 0.6722 & 0.4502
& 0.4955 & 0.3996 & 0.4291 & 0.4375
& 31.2\% \\
& Fisher
& 0.2855 & 0.2614 & 0.2583 & 0.5480
& 0.5879 & 0.5643 & 0.5360 & 0.4676
& 0.3701 & 0.3562 & 0.3426 & 0.4776
& 133.3\% \\
& Influence
& 0.7521 & 0.7276 & 0.6638 & 0.6583
& 0.2678 & 0.2418 & 0.2739 & 0.4319
& 0.4039 & 0.3996 & 0.3867 & 0.4902
& 113.9\% \\
& SSD
& 0.6388 & 0.6328 & 0.5836 & 0.6319
& 0.1837 & 0.1929 & 0.1938 & 0.4259
& 0.3396 & 0.3317 & 0.3022 & 0.5281
& 150.1\% \\
& OrthoGrad
& 0.8457 & 0.7102 & 0.6818 & 0.5235
& 0.7883 & 0.5270 & 0.5938 & 0.4197
& 0.5093 & 0.3508 & 0.3985 & 0.4107
& 45.7\% \\
& SCRUB
& 0.7685 & 0.6949 & 0.6653 & 0.5045
& 0.7397 & 0.7032 & 0.6328 & 0.4548
& 0.4806 & 0.4229 & 0.4251 & 0.4624
& 43.3\% \\
& Ours
& 0.8217 & 0.7012 & 0.6709 & 0.5483
& 0.8122 & 0.6775 & 0.6826 & 0.4213
& 0.5682 & 0.5016 & 0.5211 & 0.4285
& \textbf{14.6\%} \\
\midrule

\multirow{7}{*}{\shortstack{PU\\(SEG)}}
& NegGrad+
& 0.7682 & 0.8329 & 0.6285 & 0.5291
& 0.7745 & 0.7712 & 0.6439 & 0.5380
& 0.4930 & 0.4631 & 0.4358 & 0.5407
& 44.5\% \\
& Fisher
& 0.3098 & 0.2716 & 0.2535 & 0.5581
& 0.6667 & 0.6520 & 0.5896 & 0.5012
& 0.4683 & 0.4401 & 0.4041 & 0.5233
& 116.7\% \\
& Influence
& 0.3628 & 0.3557 & 0.3214 & 0.5622
& 0.7002 & 0.6834 & 0.6039 & 0.5036
& 0.4632 & 0.4470 & 0.4083 & 0.5271
& 104.0\% \\
& SSD
& 0.6093 & 0.6007 & 0.5631 & 0.6139
& 0.1323 & 0.1570 & 0.1482 & 0.4238
& 0.3182 & 0.3221 & 0.2907 & 0.5247
& 167.2\% \\
& OrthoGrad
& 0.4117 & 0.3556 & 0.3218 & 0.5328
& 0.7881 & 0.7357 & 0.6426 & 0.5763
& 0.5092 & 0.4768 & 0.4239 & 0.5737
& 83.4\% \\
& SCRUB
& 0.6194 & 0.5249 & 0.5193 & 0.5059
& 0.6879 & 0.6875 & 0.5839 & 0.5076
& 0.4786 & 0.4573 & 0.4138 & 0.5485
& 74.7\% \\
& Ours
& 0.8193 & 0.7003 & 0.6719 & 0.5535
& 0.8312 & 0.8427 & 0.7183 & 0.5796
& 0.5712 & 0.5928 & 0.5293 & 0.5204
& \textbf{14.0\%} \\
\midrule

\multirow{7}{*}{\shortstack{PU\\(DEP)}}
& NegGrad+
& 0.8621 & 0.8633 & 0.7305 & 0.8837
& 0.7920 & 0.6775 & 0.6739 & 0.4196
& 0.5194 & 0.4903 & 0.4412 & 0.5632
& 21.6\% \\
& Fisher
& 0.8690 & 0.8632 & 0.7289 & 0.8823
& 0.2019 & 0.1933 & 0.2649 & 0.4227
& 0.5093 & 0.4877 & 0.4372 & 0.5541
& 87.8\% \\
& Influence
& 0.7705 & 0.7651 & 0.6638 & 0.7509
& 0.3688 & 0.3201 & 0.3529 & 0.4283
& 0.4418 & 0.4286 & 0.4037 & 0.5420
& 95.6\% \\
& SSD
& 0.6183 & 0.6117 & 0.5628 & 0.6248
& 0.1592 & 0.1438 & 0.1538 & 0.4273
& 0.3193 & 0.3180 & 0.2983 & 0.5255
& 163.7\% \\
& OrthoGrad
& 0.8616 & 0.8640 & 0.7215 & 0.8956
& 0.5681 & 0.3623 & 0.4298 & 0.4237
& 0.5146 & 0.4908 & 0.4467 & 0.5975
& 59.2\% \\
& SCRUB
& 0.7327 & 0.7186 & 0.6545 & 0.5358
& 0.7996 & 0.8321 & 0.8063 & 0.4402
& 0.3477 & 0.3051 & 0.3068 & 0.4960
& 89.3\% \\
& Ours
& 0.8728 & 0.8806 & 0.7324 & 0.9316
& 0.8192 & 0.7029 & 0.6939 & 0.4319
& 0.5764 & 0.5602 & 0.5288 & 0.5118
& \textbf{9.1\%} \\
\midrule

\multirow{7}{*}{\shortstack{PU\\(NOR)}}
& NegGrad+
& 0.8609 & 0.8557 & 0.7362 & 0.8851
& 0.8114 & 0.8007 & 0.6849 & 0.5376
& 0.4559 & 0.2680 & 0.3927 & 0.4120
& 44.4\% \\
& Fisher
& 0.8537 & 0.8522 & 0.7293 & 0.8720
& 0.6291 & 0.6005 & 0.5492 & 0.4958
& 0.1629 & 0.1588 & 0.1638 & 0.4783
& 109.9\% \\
& Influence
& 0.8581 & 0.8500 & 0.7223 & 0.8522
& 0.7945 & 0.8081 & 0.6832 & 0.5023
& 0.4425 & 0.4180 & 0.4192 & 0.4832
& 42.0\% \\
& SSD
& 0.6028 & 0.6013 & 0.5640 & 0.6185
& 0.1382 & 0.1420 & 0.1438 & 0.4249
& 0.3182 & 0.3176 & 0.2830 & 0.5247
& 183.6\% \\
& OrthoGrad
& 0.8625 & 0.8519 & 0.7201 & 0.8837
& 0.8197 & 0.7622 & 0.6729 & 0.5845
& 0.1759 & 0.0687 & 0.1005 & 0.4180
& 90.1\% \\
& SCRUB
& 0.8255 & 0.8167 & 0.7021 & 0.8187
& 0.5764 & 0.5971 & 0.5239 & 0.4896
& 0.3430 & 0.2951 & 0.3258 & 0.4096
& 87.3\% \\
& Ours
& 0.8772 & 0.8925 & 0.7328 & 0.9468
& 0.8383 & 0.8506 & 0.7138 & 0.5923
& 0.5584 & 0.5033 & 0.5102 & 0.4268
& \textbf{10.6\%} \\
\bottomrule
\end{tabular}
}
\end{table*}
Table~\ref{tab:nyu_unlearning_swin} further evaluates whether the proposed interference-aware framework generalizes beyond ViT-L by replacing the backbone with Swin-L~\cite{liu2021swin}. Overall, our method consistently achieves the lowest UIS across all full-task and partial-task settings, reducing FU UIS from the strongest baseline at $31.2\%$ to $14.6\%$, and reducing the average PU UIS from $36.0\%$ to $11.2\%$, corresponding to a $68.8\%$ relative reduction. These results show that the proposed task-aware projection and instance-level orthogonalization are not specific to a particular transformer backbone.

Compared with ViT-L, Swin-L exhibits a slightly different unlearning behavior. Due to its hierarchical window-based design, Swin-L tends to produce more localized representations, which can benefit dense prediction tasks but also makes forgetting updates more sensitive to local spatial features. As a result, several baselines show larger degradation on retained segmentation (SEG) and normal (NORMAL) performance, especially under PU settings, indicating stronger interference when the forgotten task shares local spatial structures with retained tasks. In contrast, ViT-L uses global self-attention over image patches, leading to more globally mixed representations; this makes cross-task interference more diffuse but sometimes less catastrophic for individual dense tasks. Despite these architectural differences, our method remains stable on both backbones, suggesting that explicitly separating task-specific update directions and orthogonalizing retain-conflicting forget gradients is effective for both global-token and hierarchical-window representations.

\clearpage
\section{Pseudocode}\label{apx:code}
\begin{algorithm}[h]
\caption{Multi-Task Unlearning}
\label{alg:mtu}
\begin{algorithmic}[1]
\Require Pretrained weight $\mathbf{W}^\star$, forget set $\mathcal{D}_f$, retain subsets $\mathcal{D}_r^{\mathrm{clean}}$, $\mathcal{D}_r^{\mathrm{inst}}$, $\mathcal{D}_r^{\mathrm{task}}$, task set $\mathcal{T}$, learning rates $\eta_1,\eta_2$, rank $r$, subspace dimension $s$, stability constant $\varepsilon$
\Ensure Unlearned weight $\widetilde{\mathbf{W}}$

\State Initialize low-rank factors $\mathbf{A}\in\mathbb{R}^{k\times r}$ and $\mathbf{B}\in\mathbb{R}^{d\times r}$
\State Initialize task bases $\{\mathbf{U}_t\in\mathbb{R}^{r\times s}\}_{t\in\mathcal{T}}$ and projectors $\mathbf{P}_t=\mathbf{U}_t\mathbf{U}_t^\top$
\State Keep $\mathbf{W}^\star$ frozen

\For{each unlearning iteration}
    \State Sample mini-batches from $\mathcal{D}_f$, $\mathcal{D}_r^{\mathrm{clean}}$, $\mathcal{D}_r^{\mathrm{inst}}$, and $\mathcal{D}_r^{\mathrm{task}}$
    \For{each involved task $t$}
        \State Compute forget gradients $\nabla_{\mathbf{A},f}$ and $\nabla_{\mathbf{B},f}$ from $\mathcal{D}_f$
        \State Compute retain gradients $\nabla_{\mathbf{A},r}^{\mathrm{clean}}$, $\nabla_{\mathbf{A},r}^{\mathrm{inst}}$, $\nabla_{\mathbf{A},r}^{\mathrm{task}}$
        \State Compute retain gradients $\nabla_{\mathbf{B},r}^{\mathrm{clean}}$, $\nabla_{\mathbf{B},r}^{\mathrm{inst}}$, $\nabla_{\mathbf{B},r}^{\mathrm{task}}$

        \State Project all gradients using $\mathbf{P}_t$ as in Eq.~\eqref{eq:projected_gradients}

        \State Orthogonalize $\nabla_{\mathbf{A},f}$ sequentially:
        \[
        \nabla_{\mathbf{A},f}^{\perp}
        =
        \Pi_{\nabla_{\mathbf{A},r}^{\mathrm{task}}}^{\perp}
        \Bigl(
        \Pi_{\nabla_{\mathbf{A},r}^{\mathrm{inst}}}^{\perp}
        \bigl(
        \Pi_{\nabla_{\mathbf{A},r}^{\mathrm{clean}}}^{\perp}
        (\nabla_{\mathbf{A},f})
        \bigr)
        \Bigr)
        \]

        \State Orthogonalize $\nabla_{\mathbf{B},f}$ sequentially:
        \[
        \nabla_{\mathbf{B},f}^{\perp}
        =
        \Pi_{\nabla_{\mathbf{B},r}^{\mathrm{task}}}^{\perp}
        \Bigl(
        \Pi_{\nabla_{\mathbf{B},r}^{\mathrm{inst}}}^{\perp}
        \bigl(
        \Pi_{\nabla_{\mathbf{B},r}^{\mathrm{clean}}}^{\perp}
        (\nabla_{\mathbf{B},f})
        \bigr)
        \Bigr)
        \]

        \State Aggregate retain gradients:
        \[
        \nabla_{\mathbf{A},r}
        =
        \nabla_{\mathbf{A},r}^{\mathrm{clean}}
        +
        \nabla_{\mathbf{A},r}^{\mathrm{inst}}
        +
        \nabla_{\mathbf{A},r}^{\mathrm{task}},
        \quad
        \nabla_{\mathbf{B},r}
        =
        \nabla_{\mathbf{B},r}^{\mathrm{clean}}
        +
        \nabla_{\mathbf{B},r}^{\mathrm{inst}}
        +
        \nabla_{\mathbf{B},r}^{\mathrm{task}}
        \]

        \State Update $\mathbf{A}$ and $\mathbf{B}$ as in Eq.~\eqref{eq:parameter_update}
    \EndFor

    \State Regularize task subspaces by minimizing $\sum_{t\neq t'}\|\mathbf{U}_t^\top\mathbf{U}_{t'}\|_F^2$
\EndFor

\State Merge the low-rank update: $\widetilde{\mathbf{W}}=\mathbf{W}^\star+\mathbf{B}\mathbf{A}^\top$
\State \Return $\widetilde{\mathbf{W}}$
\end{algorithmic}
\end{algorithm}
% \clearpage
% \input{lib/check_list}
\end{document}